\documentclass[runningheads]{llncs}
\usepackage[T1]{fontenc}
\usepackage{graphicx}
\usepackage{booktabs}
\usepackage[misc]{ifsym}
\newcommand{\corr}{(\Letter)}
\usepackage{amsmath,amssymb,amsfonts}
\usepackage{mathtools}
\usepackage{algorithmic}
\usepackage{algorithm}
\usepackage{textcomp}
\usepackage{xurl}

\usepackage{hyperref}
\usepackage{microtype}
\usepackage{multirow}
\usepackage{subcaption}
\usepackage{xcolor}
\usepackage{enumitem}

\usepackage{array}

\newcommand{\BMmod}[1]{#1}

\graphicspath{{plots/}{../plots/}}

\newcommand{\RR}{\mathbb{R}}
\newcommand{\1}{\mathbf{1}}
\DeclareMathOperator{\diag}{diag}
\DeclareMathOperator{\rank}{rank}
\DeclareMathOperator{\tr}{tr}
\DeclareMathOperator{\KL}{KL}

\DeclareMathOperator{\Retr}{Retr}
\newcommand{\ip}[2]{\langle #1,#2\rangle}

\newcommand{\Mbal}{\mathcal{M}_{\mathrm{bal}}}
\newcommand{\Mub}{\mathcal{M}_{\mathrm{ub}}}

\makeatletter
\renewcommand\paragraph{\@startsection{paragraph}{4}{\z@}%
  {-12\p@ \@plus -4\p@ \@minus -4\p@}%
  {-0.5em \@plus -0.22em \@minus -0.1em}%
  {\normalfont\normalsize\bfseries}}
\makeatother

\begin{document}

% \title{Riemannian Optimization for\\ Low-Rank Optimal Transport}

% \titlerunning{Riemannian Optimization for Low-Rank Optimal Transport}

\title{A Riemannian Approach to\\ Low-Rank Optimal Transport}

\titlerunning{A Riemannian Approach to Low-Rank Optimal Transport}
\toctitle{A Riemannian Approach to Low-Rank Optimal Transport}   % <-- add

\author{Pratik Jawanpuria\inst{1} \corr \and Bamdev Mishra\inst{2}}

\authorrunning{Jawanpuria and Mishra}
\tocauthor{Pratik Jawanpuria, Bamdev Mishra}                     % <-- add

\institute{
Centre for Machine Intelligence and Data Science, IIT Bombay \\
\email{pratik.jawanpuria@iitb.ac.in}
\and
Microsoft India \\
\email{bamdevm@microsoft.com}
}

\maketitle

% ======================================================================
%  ABSTRACT
% ======================================================================
\begin{abstract}
Low-rank optimal transport (OT) mitigates the quadratic scaling of classical solvers, yet existing approaches rely heavily on first-order mirror-descent updates that require careful hyperparameter tuning and ignore the optimization landscape's curvature. To address these limitations, we propose a unified Riemannian geometric framework for low-rank OT, modeling balanced and unbalanced rank-$r$ positive factored couplings as novel smooth embedded submanifolds of the positive orthant. By equipping these manifolds with the Fisher-Rao product metric, we derive tractable formulations for Riemannian projectors, retractions, and Hessian-vector products. Our cost-agnostic framework seamlessly extends to linear OT, Gromov-Wasserstein (GW), fused GW, and their unbalanced counterparts. For balanced OT, our geometric ingredients are computed via efficient conjugate-gradient and iterative Bregman updates. For the unbalanced OT, our operations elegantly reduce to closed-form scalings, completely eliminating inner iterative loops. In both regimes, per-iteration computations scales linearly with dataset size. For convex transport costs, we further derive a rank-sufficiency certificate with a computable duality-gap bound
for certifying global optimality. Across a range of problem sizes, our regularization-free first- and second-order solvers converge faster to competitive or better solutions than state-of-the-art low-rank OT solvers.

\end{abstract}

% ======================================================================
%  1. INTRODUCTION
% ======================================================================
\section{Introduction}
\label{sec:intro}
Optimal transport (OT) provides a mathematically rigorous framework for comparing and aligning probability distributions~\cite{Villani2009,PeyreBook2019}. At its core, the classical Kantorovich formulation~\cite{kantorovich1942transfer} seeks a joint distribution -- or \textit{coupling} -- that matches a source measure to a target measure while minimizing the total expected cost of transporting mass between them. Since it explicitly incorporates the geometry of the underlying space through a ground cost function, OT defines a metric (the Wasserstein distance) that handles non-overlapping supports better than information-theoretic divergences like the Kullback-Leibler divergence. Consequently, OT has established itself as a cornerstone technique across a wide spectrum of machine learning applications, driving critical advances in supervised and unsupervised learning~\cite{frogner15a,Vayer2019,jawanpuria21a}, generative modeling \cite{arjovsky2017wgan,Genevay2018LearningGM}, domain adaptation \cite{Courty2017OTDA,Courty17domAda,nath20a,manupriya24a,jawanpuria25a}, natural language processing~\cite{alvarez18wordEmb,jawanpuria20a}, shape interpolation~\cite{peyre2016gromov,han22a}, and targeted representative sampling~\cite{gurumoorthy21a,chanda26a}. 
Much of this widespread adoption was catalyzed by entropic regularization~\cite{Cuturi2013}, which smooths the (linear) OT problem and admits the highly parallelizable Sinkhorn algorithm~\cite{knight2008sinkhorn}. This enabled OT to be employed as a practical loss function for large-scale learning.
%Optimal transport (OT) provides a principled way to compare and couple probability distributions~\cite{Villani2009,PeyreBook2019}, with applications spanning control~\cite{ChenGeorgiouPavon2021}, estimation~\cite{PeyreBook2019}, machine learning~\cite{Cuturi2013,Flamary2021}, and robotics~\cite{Solomon2015}. 

However, even with entropic regularization, OT retains an $O(mn)$ bottleneck as it requires  computing a dense coupling matrix between $m$ source and $n$ target points.
% However, classical discrete OT scales quadratically, requiring an $O(mn)$ memory and computation to construct a dense coupling matrix between $m$ source and $n$ target points. 
To circumvent this, a natural remedy is to constrain the transport plan to a low-dimensional structure. In this regard, one line of work seeks a \textit{sparse} transport plan, in which each source routes mass to only a few targets so that the coupling has far fewer than $mn$ non-zero entries. Since the entropic plan is inherently dense, this is achieved by replacing the entropy with a squared-$\ell_2$ or (structured) sparse cardinality penalty~\cite{Blondel2018SmoothAS,manupriya24b}. This is motivated by the classical result that the unregularized OT problem always admits a sparse optimal solution (a vertex of the transportation polytope) with at most $m+n-1$ non-zero entries~\cite{PeyreBook2019}. 
A complementary line of work, which we pursue here, constrains the transport plan globally rather than entry-wise. Instead of limiting how many entries are non-zero, it limits the rank of the coupling matrix, requiring all mass to be routed through a small set of $r\ll \min(m,n)$ shared latent hubs: each source first sends its mass to these hubs, which in turn redistribute it to the targets~\cite{Forrow2019,ScetbonCuturiPeyre2021}. This structural prior reduces the parameter space to $O(r(m+n))$ while naturally extending to the quadratic Gromov-Wasserstein (GW) setting~\cite{ScetbonPeyreCuturi2022} and unbalanced OT regimes where strict marginal constraints are relaxed~\cite{ScetbonSejourne2023}.

Despite their conceptual elegance, optimizing these low-rank balanced and unbalanced OT formulations remains challenging. State-of-the-art solvers such as LOT~\cite{ScetbonCuturiPeyre2021}, LR-GW~\cite{ScetbonPeyreCuturi2022}, and UB-LOT~\cite{ScetbonSejourne2023} rely heavily on alternating mirror descent updates coupled with Dykstra or Bregman projections. 
%Alternatively, Frank--Wolfe schemes like FRLC~\cite{Halmos2024} utilize log-domain Sinkhorn sub-routines (\alertbyPJ{which sacrifices computational efficiency for numerical robustness}). 
Consequently, they are susceptible to ill-conditioning, require careful tuning of step-sizes or entropic regularization, and cannot exploit the local curvature of the optimization landscape to accelerate convergence.
% Overall, the existing low-rank optimal transport  approaches are inherently first-order methods. Hence, they are susceptible to ill-conditioning, often require careful tuning of step-sizes or entropic regularization, and crucially, cannot exploit the curvature of the optimization landscape to accelerate convergence. 

%Existing methods---LOT~\cite{ScetbonCuturiPeyre2021} for linear OT, LR-GW~\cite{ScetbonPeyreCuturi2022} for GW, and UB-LOT~\cite{ScetbonSejourne2023} for unbalanced OT---use mirror descent with Dykstra projections 
% \alertbyPJ{they can also use Bregman projections}; 

% all are fast but enforce marginals only approximately \alertbyPJ{This enforcing marginal constraint should not be discussed} and are limited to first-order updates. FRLC~\cite{Halmos2024} uses Frank--Wolfe with log-domain Sinkhorn sub-problems. None of these methods can exploit curvature (Hessian) information.

In this work, we bridge this gap by introducing a unified \textit{Riemannian geometric framework} for low-rank OT. We establish that the set of rank-$r$ positive factored couplings forms a smooth embedded submanifold $\mathcal{M}$ of the positive orthant. By equipping $\mathcal{M}$ with the Fisher-Rao product metric~\cite{AmariNagaoka2000,Modin2017},  we derive explicit, tractable formulations for Riemannian projectors, retractions, and Hessian-vector products. This enables robust first- and second-order solvers, such as conjugate gradients (CG) and trust-regions (TR)~\cite{AbsilMahonySepulchre2008,Boumal2023}. 
% This specific geometric choice unlocks multiple computational advantages: it yields explicit, tractable formulations for Riemannian projectors, retractions, and Hessian-vector products. Consequently, we enable the deployment of robust first- and second-order solvers, such as conjugate gradients (CG) and trust-regions (TR) \cite{AbsilMahonySepulchre2008,Boumal2023}, to the low-rank OT problem landscape.
% 
% We take a {geometric} approach. We show that the set of \alertbyPJ{it is technically $\leq r$?} rank-$r$ positive factored couplings forms a smooth embedded submanifold $\mathcal{M}$ of a positive orthant, in both the balanced (prescribed marginals) and unbalanced (penalized marginals) settings. Equipping $\mathcal{M}$ with the Fisher--Rao product metric~\cite{AmariNagaoka2000,Modin2017} yields all ingredients needed for Riemannian optimization---projector, retraction, gradient, and Hessian-vector products---enabling steepest descent (SD), conjugate gradient (CG), and trust-region (TR) solvers~\cite{AbsilMahonySepulchre2008,Boumal2023}. All manifold ingredients are {cost-agnostic}: switching between linear OT, GW, fused GW, and unbalanced variants requires changing only the Euclidean gradient and Hessian-vector product.
% 
% \paragraph{Contributions.}
The main contributions of this work are summarized as follows.
\begin{enumerate}[nosep]
    \item \textbf{Riemannian formulation for low-rank OT:} We formalize the feasible set of positive factored couplings as a smooth manifold for both balanced and unbalanced OT settings, deriving all necessary first- and second-order Riemannian ingredients under the Fisher-Rao metric. 
  \item \textbf{Closed-form expressions for unbalanced OT:} We show that our tangent space projectors and retractions for unbalanced OT reduce to closed-form expressions, entirely eliminating the inner Dykstra/Bregman loops required by existing approaches~\cite{ScetbonSejourne2023}.
  \item \textbf{A cost-agnostic framework:} Our geometric ingredients are independent of the ground cost. Transitioning between linear OT, GW, fused GW, or their unbalanced variants simply requires modifying the Euclidean gradient and Hessian-vector product.
  %\item The framework is cost-agnostic: linear OT, GW, fused GW, and their unbalanced variants all share the same manifold machinery (Sect.~\ref{sec:costs}).
  \item \textbf{Empirical efficiency:} 
   Experiments up to $n=50\,000$ demonstrate that our Riemannian solvers converge in an order-of-magnitude fewer outer iterations than mirror-descent approaches, achieving state-of-the-art performance without step-size tuning. 
  % Through extensive experiments up to $n=50\,000$, we demonstrate that our Riemannian solvers converge in an order-of-magnitude fewer outer iterations than existing mirror-descent based approaches, without requiring step-size tuning or entropic regularization. For unbalanced OT, this translates to state-of-the-art performance.
  %\item Experiments at scales up to $n=50\,000$ show competitive cost, with Riemannian solvers requiring far fewer outer iterations than mirror-descent baselines. For balanced OT, wall-clock time is sometimes comparable or even faster for LOT, but in the unbalanced case, the Riemannian approach is consistently much faster in both iteration count and wall-clock time. No entropic regularization or step-size tuning is needed (Sect.~\ref{sec:experiments}). \alertbyPJ{Discussion needs to improve. Currently, the last point looks every verbose and not at the level of NeurIPS.}
\end{enumerate}
% Proofs of the theoretical results, experimental details, and additional experiments are provided in the extended version~\cite{jawanpuria26a}.

\section{Preliminaries and Related works}\label{sec:related}
\paragraph{Optimal transport and entropic regularization.} Let $\Delta_k = \{p \in \mathbb{R}^k_{++} \mid p^\top \mathbf{1}_k = 1\}$ denote the probability simplex. Consider a discrete source distribution supported on $m$ points $\{x_i\}_{i=1}^m$ with associated probability weights $a \in \Delta_m$, and a discrete target distribution supported on $n$ points $\{y_j\}_{j=1}^n$ with weights $b \in \Delta_n$, with $x_i,y_j\in\RR^d$. 
Given a cost matrix $C \in \mathbb{R}^{m \times n}$, where $C_{ij}$ represents the ground cost of moving a unit of mass from $x_i$ to $y_j$, the classical discrete Kantorovich optimal transport (OT) problem seeks a coupling $\Gamma$ that minimizes the total transport cost:
\begin{equation}\label{eqn:linearOT}
    \begin{array}{c}
         \min_{\Gamma \in \Omega(a,b)} \langle C, \Gamma \rangle,
    \end{array}
\end{equation}
% $$\min_{\Gamma \in \Pi(a,b)} \langle C, \Gamma \rangle,$$
where $\Omega(a,b) = \{\Gamma \in \mathbb{R}^{m \times n}_+ \mid \Gamma \mathbf{1}_n = a, \Gamma^\top \mathbf{1}_m = b\}$ is the set of joint distributions between $a$ and $b$. 
%Solving this linear program scales as $O(n^3 \log n)$ when $m=n$, which is computationally prohibitive for large datasets. 
Solving the above linear program (\ref{eqn:linearOT}) is computationally prohibitive for large datasets. Hence, Cuturi~\cite{Cuturi2013} introduced entropic regularization $H(\Gamma) = -\sum_{i,j} \Gamma_{ij} (\log \Gamma_{ij} - 1)$ to the objective:
\begin{equation}\label{eqn:entropicOT}
    \begin{array}{c}
         \min_{\Gamma \in \Omega(a,b)} \langle C, \Gamma \rangle - \varepsilon H(\Gamma).
    \end{array}
\end{equation}
% $$\min_{\Gamma \in U(a,b)} \langle C, \Gamma \rangle - \varepsilon H(\Gamma).$$
The entropic regularization smooths the OT objective and yields a unique optimal coupling of the form $\Gamma = \mathrm{diag}(u) K \mathrm{diag}(v)$, where $u\in\RR^m,v\in\RR^n,K\in\RR^{m\times n}, $ and $K_{ij} = \exp(-C_{ij}/\varepsilon)$. The scaling vectors $u$ and $v$ can be efficiently computed using the Sinkhorn algorithm \cite{knight2008sinkhorn}. While highly parallelizable, the Sinkhorn algorithm still requires computing and storing the dense matrix $K$, leaving a $O(mn)$ bottleneck in memory and computation. Efficient low-rank approximations of $K$ have been proposed in this regard~\cite{altschuler2019massively}.

\paragraph{Low-rank OT.}
Beyond computational scaling, classical OT suffers from a severe curse of dimensionality: in $d$-dimensional spaces, the empirical Wasserstein distance converges to its population counterpart at a slow rate of $O(n^{-1/d})$. To circumvent both the computational and statistical bottlenecks, Forrow et al.~\cite{Forrow2019} introduced factored couplings. Geometrically, this approach posits that mass transport between distributions occurs through a small number of latent ``hubs'' or modes, effectively projecting the transport plan onto a lower-dimensional structure. 
%This structural prior not only improves sample complexity and generalization guarantees but also yields highly interpretable transport plans. 
Building on \cite{Forrow2019}, Scetbon et al.~\cite{ScetbonCuturiPeyre2021} formalized the search space of couplings with marginals $a\in\Delta_m$ and $b\in\Delta_n$ and a bounded non-negative rank $r \ll \min(m,n)$ as $\Omega_r(a,b) = \{ \Gamma \in \mathbb{R}^{m \times n}_+ \mid \Gamma \mathbf{1} = a,\, \Gamma^\top \mathbf{1} = b,\, \mathrm{rank}_+(\Gamma) \leq r \}$. Any such coupling $\Gamma \in \Omega_r(a,b)$ can be explicitly factored~\cite{cohen1993nonnegative} as: 
\begin{equation}\label{eqn:gamma}
    \begin{array}{c}
    \Gamma = U \mathrm{diag}(g)^{-1} V^\top,
    \end{array}
\end{equation}
where $U \in \mathbb{R}^{m \times r}_+$ and $V \in \mathbb{R}^{n \times r}_+$ are sub-couplings that share a strictly positive latent marginal $g = U^\top \mathbf{1} = V^\top \mathbf{1} \in \mathbb{R}^r_{++}$. Here, $\mathrm{diag}(g)$ is the diagonal matrix corresponding to the vector $g$. 
%This factorization neatly isolates the low-rank structure. The exact low-rank OT formulation is then dictated purely by the treatment of the outer marginals. 
In the balanced setting~\cite{ScetbonCuturiPeyre2021,ScetbonPeyreCuturi2022}, the sub-couplings must strictly reconstruct the data marginals, imposing the hard constraints $U \mathbf{1}_r = a$ and $V \mathbf{1}_r = b$. Conversely, in the unbalanced setting~\cite{ScetbonSejourne2023}, these strict equalities are dropped, allowing the row sums of $U$ and $V$ to deviate from $a$ and $b$ at the cost of Kullback--Leibler (KL) divergence penalties added to the objective~\cite{ChizatPSV18}. 
%Motivated by this structural dichotomy, we now formally define the low-rank coupling manifolds for both the balanced and unbalanced regimes.

% Building on \cite{Forrow2019}, Scetbon et al.~\cite{ScetbonCuturiPeyre2021} formulated the low-rank optimal transport problem by explicitly parameterizing the coupling through a set of $r \ll \min(m,n)$ latent modes. Specifically, the coupling is decomposed as $\Gamma = U \mathrm{diag}(g)^{-1} V^\top,$ where $U \in \RR^{m \times r}_+$ and $V \in \RR^{n \times r}_+$ represent the assignments from the source and target points to the latent modes, respectively, and $g = U^\top \mathbf{1} = V^\top \mathbf{1} \in \mathbb{R}^r_+$ represents the marginal mass distribution across these modes. This positive factorization reduces the degrees of freedom from $mn$ to $r(m+n)$. 

Despite the structural advantages of this factorization, the corresponding optimization remains challenging. To address this, Scetbon et al.~\cite{ScetbonCuturiPeyre2021} proposed the Low-Rank Optimal Transport (LOT) algorithm, which employs a mirror descent (MD) optimization scheme in its outer loop and relies on Dykstra's algorithm or iterative Bregman projections to solve the inner updates. This two-factor MD architecture serves as the foundation for several extensions, including LR-GW~\cite{ScetbonPeyreCuturi2022} for the quadratic Gromov--Wasserstein problem, and UB-LOT~\cite{ScetbonSejourne2023} for the unbalanced setting. The MD-based paradigm also extends to the recently proposed FRLC algorithm by Halmos et al.~\cite{Halmos2024}. FRLC decouples the transport plan into three matrices, rather than the two-factor form above. Notably, existing low-rank OT approaches are first-order methods. Hence, they are susceptible to ill-conditioning and are unable to leverage local curvature information, motivating the need for a geometric perspective.

% Forrow et al.~\cite{Forrow2019} introduced factored couplings $\Gamma = UDV^\top$ for Wasserstein barycenters \alertbyPJ{I dont think its technically correct}, showing that rank-$r$ structure is both statistically and computationally beneficial.
% Scetbon, Cuturi \& Peyr\'e~\cite{ScetbonCuturiPeyre2021} proposed LOT (Low-Rank Sinkhorn), which optimizes factored couplings via mirror descent with Dykstra inner projections \alertbyPJ{or Bregman}; LOT is fast and GPU-friendly \alertbyPJ{need not highlight GPU} but is limited to first-order updates.
% Scetbon et al.~\cite{ScetbonPeyreCuturi2022} extended LOT to GW (LR-GW) using the same mirror-descent structure; LR-GW requires dense $n{\times}n$ intra-cost matrices, limiting it to moderate~$n$ \alertbyPJ{What is this intra-cost matrices?}. 
% Halmos et al.~\cite{Halmos2024} introduced FRLC, a Frank--Wolfe method with latent-coupling factorizations and log-domain Sinkhorn sub-problems; FRLC optimizes an entropy-regularized surrogate, yielding higher unregularized cost but exact marginals in the limit of its inner solver.
% Scetbon et al.~\cite{ScetbonSejourne2023} addressed the unbalanced setting with mirror descent.
% All these methods are limited to first-order updates.

\paragraph{Riemannian optimization on matrix manifolds.} Riemannian optimization generalizes unconstrained optimization to smooth manifolds~\cite{AbsilMahonySepulchre2008,Boumal2023}. By equipping a manifold $\mathcal{M}$ with a Riemannian metric (an inner product on the tangent space $T_x \mathcal{M}$), one can define gradients and Hessians intrinsically. Iterative algorithms then proceed by computing a search direction in the tangent space and using a retraction mapping $\Retr_x : T_x \mathcal{M} \to \mathcal{M}$ to pull the update back onto the manifold $\mathcal{M}$. While the geometry of structured matrices is well-explored -- such as Riemannian preconditioners for standard low-rank matrices~\cite{MishraSepulchre2016} or  the geometric study of doubly stochastic matrices  \cite{DouikHassibi2019,shi2021coupling} -- the choice of metric is critical. Euclidean metrics often lead to poor conditioning near the boundaries of positive orthants and probability simplices. Conversely, the Fisher-Rao metric naturally captures the intrinsic geometry of multinomial manifolds~\cite{sun2015heterogeneous}. By equipping the space of low-rank factored couplings with a Fisher-Rao product metric, we derive explicit retractions and Hessian-vector products that bypass the limitations of first-order mirror descent.

% In this work, we synthesize these perspectives: we equip the space of low-rank factored couplings with a Fisher-Rao product metric, allowing us to compute explicit closed-form retractions and Hessian-vector products that bypass the limitations of first-order mirror descent.

% Optimization on matrix manifolds~\cite{AbsilMahonySepulchre2008,Boumal2023} has been applied to structured problems including low-rank factorizations with preconditioning~\cite{MishraSepulchre2016} and doubly stochastic matrices~\cite{DouikHassibi2019}.
% Sun et al.~\cite{sun2015heterogeneous} equipped the multinomial manifold with the Fisher information metric and derived second-order ingredients for a trust-region solver in tensor clustering.
% Douik \& Hassibi~\cite{DouikHassibi2019} extended this geometry to doubly stochastic matrices; our construction generalizes further to the three-block coupling structure $(U,V)$.

% \paragraph{Entropic OT.}
% Sinkhorn-based methods~\cite{Cuturi2013,Altschuler2017,Benamou2015} regularize the full-rank problem; our approach is complementary, controlling rank directly rather than smoothness. \alertbyPJ{Could be merged with low-rank OT para. Requires better discussion.}

% ======================================================================
%  2. LOW-RANK COUPLING MANIFOLD
% ======================================================================
\section{The Geometry of Low-Rank Couplings} %Low-Rank Coupling Manifold 
\label{sec:manifold} 
As discussed in Section~\ref{sec:related}, a low-rank coupling $\Gamma$ may be modeled in terms of sub-couplings ($U,V$) that share strictly positive latent marginal distribution (\ref{eqn:gamma}). This factorization neatly isolates the low-rank structure. The exact search space of low-rank couplings is then dictated purely by the treatment of the outer marginals ($a$ and $b$), which leads to the balanced and unbalanced OT settings. Motivated by this structural dichotomy, we now formally define the low-rank coupling manifolds for both the regimes.

\subsection{Low-rank Coupling Manifolds}
Let $\mathcal{A} := \RR_{++}^{m \times r} \times \RR_{++}^{n \times r}$ denote the ambient positive orthant. Building on the factorization (\ref{eqn:gamma}), we propose the following low-rank coupling manifolds for the balanced and unbalanced OT settings: 
\begin{equation}\label{eq:Mbal} %\label{eq:Mub}
\begin{array}{ll}
  \Mbal &\coloneqq \bigl\{(U,V) \in \mathcal{A} : U\1{=}a,\; V\1{=}b,\; U^\top\!\1{=}V^\top\!\1 \bigr\}, \\
  \Mub  &\coloneqq \bigl\{(U,V) \in \mathcal{A} : U^\top\!\1{=}V^\top\!\1 \bigr\}. 
\end{array}
\end{equation}
The balanced manifold $\mathcal{M}_{\mathrm{bal}}$ strictly enforces both the inner latent consistency and the outer data marginals. In contrast, the unbalanced manifold $\mathcal{M}_{\mathrm{ub}}$ relaxes the outer marginal constraints entirely, requiring only that the sub-couplings agree on the latent mass distribution. This yields the natural inclusion $\mathcal{M}_{\mathrm{bal}} \subset \mathcal{M}_{\mathrm{ub}} \subset \mathcal{A}$.  
Table~\ref{tab:manifold_comparison} summarizes the structural differences between the two geometries.

% \begin{table}[t]
% \centering
% \caption{Balanced vs. unbalanced manifold: constraints, dimension, and tangent space.}
% \label{tab:manifold_comparison}
% \small
% \begin{tabular}{@{}l |l|l@{}}
% \toprule
% & \textbf{Balanced $\Mbal$} & \textbf{Unbalanced $\Mub$} \\
% \midrule
% Equality constraints & $U\1{=}a$, $V\1{=}b$, $U^\top \1{=}V^\top\1$ & $U^\top \1{=}V^\top\1$ \\
% % \# scalar constraints\ \  & $m{+}n{+}r{-}1$ (one redundancy) & $r$ (no redundancy) \\
% Dimension & $(m{+}n{-}1)(r{-}1)$ & $(m{+}n{-}1)\,r$ \\
% % $r{=}1$ & Degenerate (single point $ab^\top$) & $(m{+}n{-}1)$-dimensional \\
% % \midrule
% Tangent space & $\dot U\1{=}0$, $\dot V\1{=}0$, ${\dot U}^\top\1{=}{\dot V}^\top\1$ & ${\dot U}^\top\1{=}{\dot V}^\top\1$ \\
% % Marginals & Fixed: $\Gamma\1{=}a$, $\Gamma^\top\!\1{=}b$ & Free: $\Gamma\1{=}U\1$, $\Gamma^\top\!\1{=}V\1$ \\
% \bottomrule
% \end{tabular}
% \end{table}

\begin{theorem}[Manifold structure]\label{thm:manifold_bal}
Let $a \in \Delta_m$, $b \in\Delta_n$, and $r \geq 2$.
\begin{enumerate}[nosep]
  \item $\Mbal$ is a non-empty, embedded smooth submanifold of $\mathcal{A}$ with $\dim(\Mbal) = (m{+}n{-}1)(r{-}1)$.
  \item $\Mub$ is a non-empty, embedded smooth submanifold of $\mathcal{A}$ with $\dim(\Mub) = (m{+}n{-}1)\,r$.
\end{enumerate}
\end{theorem}

\paragraph{Tangent space.} The tangent spaces at a point $(U,V)$ are obtained by differentiating the respective constraint equations, i.e.,
% yielding the kernels of the constraint Jacobians:
\begin{equation}\label{eq:tangent_bal} %\label{eq:tangent_ub}
\begin{array}{ll}
  T_{(U,V)}\Mbal &= \{(\dot U, \dot V) \in \mathbb{R}^{m \times r} \times \mathbb{R}^{n \times r} \mid \dot U\1{=}0,\; \dot V\1{=}0,\; {\dot U}^\top\1{=}{\dot V}^\top\1 \}, \\
  T_{(U,V)}\Mub  &= \{(\dot U, \dot V) \in \mathbb{R}^{m \times r} \times \mathbb{R}^{n \times r} \mid {\dot U}^\top\1{=}{\dot V}^\top\1 \}. 
\end{array}
\end{equation}
We note that $T_{(U,V)}\Mbal \subset T_{(U,V)}\Mub$. 
% The unbalanced tangent space is strictly larger: it restricts only the column sums of the tangent vectors to match, placing no restrictions on their row sums.
%The unbalanced tangent space is strictly larger: it does not constrain the row sums of $\dot U$ or $\dot V$, only their column sums. The dimension gap $(m{+}n{-}1)$ between $\Mub$ and $\Mbal$ corresponds to the freed marginals.

%\label{sec:geometry}

%We equip both manifolds with the same Fisher--Rao product metric and derive the Riemannian ingredients, highlighting the dramatic simplifications in the unbalanced case.

\subsection{Fisher-Rao Product Metric}
To perform Riemannian optimization, we need to equip the ambient space $\mathcal{A}$ with a Riemannian metric. Given the positivity constraints on our factors, the Euclidean metric often leads to poor conditioning near the boundary. Instead, we adopt the Fisher-Rao metric, which naturally captures the information geometry of positive orthants~\cite{AmariNagaoka2000,Modin2017}. We define the product metric on the ambient space $\mathcal{A}$ at a point $(U,V)$ for tangent vectors $\xi = (\xi_U, \xi_V)$ and $\eta = (\eta_U, \eta_V)$ as:
%On $\RR_{++}^d$, the Fisher--Rao metric is $g_x(\xi,\eta) = \sum_i \xi_i \eta_i / x_i$. The product metric on $(U,V)$ is
\begin{equation}
\langle \xi, \eta \rangle_{(U,V)} := \sum_{i,k} \frac{(\xi_U)_{ik} (\eta_U)_{ik}}{U_{ik}} + \sum_{j,k} \frac{(\xi_V)_{jk} (\eta_V)_{jk}}{V_{jk}}.
  \label{eq:metric}
\end{equation}
We note that the Fisher-Rao metric corresponds to the Hessian of the Kullback-Leibler (KL) divergence and it yields highly tractable geometric operations for our problem setting, as studied in the next sections. 

\subsection{Orthogonal Tangent Projector and Retraction}
Riemannian optimization requires two fundamental operations: an orthogonal projection $\Pi Z$ to map ambient vectors $Z = (Z_U, Z_V)$ onto the tangent space, and a retraction $\Retr_{(U,V)} : T_{(U,V)}\mathcal{M} \to \mathcal{M}$ to pull tangent updates back onto the manifold. Both ingredients take qualitatively different forms depending on whether the outer marginal constraints are strictly enforced. Table~\ref{tab:geom_comparison} summarizes these operations, highlighting the computational simplifications achieved in the unbalanced setting.

\paragraph{Balanced projector.} Let $\odot$ and $\oslash$ denote the Hadamard (element-wise) product and division. Under the Fisher--Rao metric, the orthogonal projection $\Pi Z$ onto $T_{(U,V)}\mathcal{M}_{\mathrm{bal}}$ removes the normal component:
\begin{equation}\label{eq:projUV_bal}
\begin{array}{l}
  (\Pi Z)_U = Z_U - U \odot (\alpha\1^\top + \1\gamma^\top), \ \ \  (\Pi Z)_V = Z_V - V \odot (\beta\1^\top - \1\gamma^\top), 
\end{array}
\end{equation}
where the Lagrange multipliers $\alpha \in \mathbb{R}^m$, $\beta \in \mathbb{R}^n$, and $\gamma \in \mathbb{R}^r$ are found by solving the KKT system:
\begin{equation}
  \begin{bmatrix} \diag(a) & 0 & U \\ 0 & \diag(b) & -V \\ U^\top & -V^\top & 2\diag(s) \end{bmatrix}
  \begin{bmatrix} \alpha \\ \beta \\ \gamma \end{bmatrix}
  =
  \begin{bmatrix} Z_U \1 \\ Z_V \1 \\ \1^\top\! Z_U - \1^\top\! Z_V \end{bmatrix},
  \label{eq:KKT}
\end{equation}
with $s = U^\top\1 = V^\top\1$. While this system has a rank deficiency of one (as the global sum constraint is redundant), the projection $\Pi Z$ itself is unique. By fixing the gauge $\gamma^\top \mathbf{1} = 0$, the block system can be reduced to solving an $(r{-}1) \times (r{-}1)$ positive-definite Schur complement $S = 2\mathrm{diag}(s) - U^\top \mathrm{diag}(a)^{-1} U - V^\top \mathrm{diag}(b)^{-1} V$.

%The system has rank $m{+}n{+}r{-}1$, but $\Pi Z$ is unique (the null vector $(-\1_m, \1_n, \1_r)$ maps to zero correction). Fixing the gauge $\gamma^\top\!\1 = 0$ yields a reduced $(r{-}1) \times (r{-}1)$ positive-definite Schur complement $S = 2\diag(s) - U^\top \diag(a)^{-1} U - V^\top \diag(b)^{-1} V$.

\begin{remark}\label{rem:cg}
The Schur system can be solved iteratively via conjugate gradients at cost $O(k_{\mathrm{CG}}\,r(m{+}n))$ per projection.
Since the condition number $\kappa(S_{\mathrm{red}})$ remains small in practice (empirically 1.2--1.3; see Appendix~\ref{sec:app_retraction}), $k_{\mathrm{CG}}$ is small.
\end{remark}

\paragraph{Unbalanced projector.} We note that relaxing the outer marginals decouples the row constraints. Hence, the orthogonal projection onto $T_{(U,V)}\mathcal{M}_{\mathrm{ub}}$ simplifies to a pure closed-form evaluation: 
%The projection onto $T_{(U,V)}\Mub$~\eqref{eq:tangent_ub} is
\begin{equation}
  (\Pi Z)_U = Z_U - U \odot (\1\gamma^\top) \textup{\ \ and\ \ }
  (\Pi Z)_V = Z_V + V \odot (\1\gamma^\top), 
  \label{eq:proj_ub}
\end{equation}
where $\gamma = (Z_U^\top\1 - Z_V^\top\1)\oslash(2s)$. 
No Schur complement, gauge fixing, or iterative solver is needed, reducing the exact projection cost to $O(r(m{+}n))$.

\begin{table}[t]
\centering
\caption{Balanced vs. unbalanced manifold: constraints, dimension, tangent space, and Riemannian ingredients. $N = m{+}n$; $k_{\mathrm{CG}}$: Schur CG iterations; $L_{\mathrm{sink}}$: Bregman iterations; $\odot$ and $\oslash$ denote element-wise product and division.}
\label{tab:manifold_comparison}
\label{tab:geom_comparison}
\small
\begin{tabular}{@{}l|c|c@{}}
\toprule
& \textbf{Balanced $\Mbal$} & \textbf{Unbalanced $\Mub$} \\
\midrule
Eq. constraints & $U\1{=}a$, $V\1{=}b$, $U^\top \1{=}V^\top\1$ & $U^\top \1{=}V^\top\1$ \\
% \# scalar constraints\ \  & $m{+}n{+}r{-}1$ (one redundancy) & $r$ (no redundancy) \\
Dimension & $(m{+}n{-}1)(r{-}1)$ & $(m{+}n{-}1)\,r$ \\
% $r{=}1$ & Degenerate (single point $ab^\top$) & $(m{+}n{-}1)$-dimensional \\
% \midrule
Tangent space & $\{(\dot U,\dot V)\mid\dot U\1{=}0$, $\dot V\1{=}0$, ${\dot U}^\top\1{=}{\dot V}^\top\1\}$ & $\{(\dot U,\dot V)\mid{\dot U}^\top\1{=}{\dot V}^\top\1\}$ \\
\midrule
\textbf{Projector} &  \\
\multirow{2}{*}{Formula} & $(\Pi Z)_U = Z_U - U \odot (\alpha\1^\top {+} \1\gamma^\top)$ (\ref{eq:projUV_bal}) & $(\Pi Z)_U = Z_U - U \odot (\1\gamma^\top)$ (\ref{eq:proj_ub}) \\
& $(\Pi Z)_V = Z_V - V \odot (\beta\1^\top - \1\gamma^\top)$ (\ref{eq:projUV_bal})& $(\Pi Z)_V = Z_V + V \odot (\1\gamma^\top)$ (\ref{eq:proj_ub}) \\
\cmidrule{2-3}
\multirow{2}{*}{Solver} & CG solver for  the Schur system & $\gamma = (Z_U^\top\1 {-}  Z_V^\top\1)\oslash(2s)$,  \\
 & involving $\alpha,\beta,\gamma$ (\ref{eq:KKT}) & Closed-form expression \\
 \cmidrule{2-3}
 Complexity & $O(k_{\mathrm{CG}}\,rN)$, $k_{\mathrm{CG}}$ inner iterations    & $O(rN)$ \\
\midrule
\textbf{Retraction}  \\
 \multirow{2}{*}{Mirror step} & \multicolumn{2}{c}{$\bar{U} = U \odot \exp(t\,\Xi_U \oslash U)$, \quad $\bar{V} = V \odot \exp(t\,\Xi_V \oslash V)$} \\
 & \multicolumn{2}{c}{(shared formula for $\Mbal,\Mub$)} \\
\cmidrule{2-3}
 KL projection & Cyclic Bregman (Algorithm~\ref{alg:sinkhorn})  & Single column scaling~\eqref{eq:retr_ub} \\
\cmidrule{2-3}
 Complexity & $O(L_{\mathrm{sink}}\,rN)$, $L_{\mathrm{sink}}$ inner iterations   & $O(rN)$\\
\midrule
\textbf{Gradient} & \multicolumn{2}{c}{$\mathrm{grad}\,f = \Pi(U \odot \nabla_U f,\; V \odot \nabla_V f)$ (\ref{eq:rgrad}), (different $\Pi$ for $\Mbal,\Mub$)} \\
\midrule
\textbf{Hessian-vec} & \multicolumn{2}{c}{$\mathrm{Hess}\,f[\Xi] = \Pi(\tfrac{1}{2}\Xi \odot \nabla_E f + W \odot \nabla_E^2 f[\Xi])$ (\ref{eq:hess_impl}), (different $\Pi$)} \\
\midrule
\textbf{Overall per-} & \multirow{2}{*}{$O((k_{\mathrm{CG}}{+}L_{\mathrm{sink}})\,rN)$ }& \multirow{2}{*}{${O(rN)}$ (no inner iterations)} \\
\textbf{iteration cost} & \\
\bottomrule
\end{tabular}
\end{table}

\paragraph{Retraction.} 
A valid retraction must map a tangent vector $t\Xi$ back to the manifold while preserving the strictly positive structure of the factors. For both manifolds, we achieve this via a two-step procedure. First, we apply a Fisher-consistent exponential mirror step: 
\begin{equation}
  \bar{U} = U \odot \exp(t\,\Xi_U \oslash U), \quad \bar{V} = V \odot \exp(t\,\Xi_V \oslash V).
  \label{eq:mirror}
\end{equation}
Second, we project the intermediate factors $(\bar{U}, \bar{V})$ onto the respective manifold by minimizing the KL divergence:  ${\KL(U \| \bar{U}) + \KL(V \| \bar{V})}$. 
%preserves strict positivity and is first-order consistent with the Fisher geometry. The resulting $(\bar{U}, \bar{V})$ is then projected back to $\mathcal{M}$ by minimizing $\KL(U \| \bar{U}) + \KL(V \| \bar{V})$.

For $\Mbal$, this requires cyclic Bregman projections onto $\{U\1{=}a\}$, $\{V\1{=}b\}$, and  $\{s(U)=s(V)\}$, where $s(U) = U^\top\!\1$ and $s(V) = V^\top\!\1$. This is detailed in Algorithm~\ref{alg:sinkhorn}, which converges linearly \cite{luo1992convergence}. 
However, for $\mathcal{M}_{\mathrm{ub}}$, the KL projection has a closed-form column expression:
\begin{equation}
  U^+ = \bar{U}\,\diag(t), \qquad V^+ = \bar{V}\,\diag(t^{-1}), \qquad t = \sqrt{s(\bar{V}) \oslash s(\bar{U})},
  \label{eq:retr_ub}
\end{equation}
where $s(\bar{U}) = \bar{U}^\top\!\1 \in \RR^r$ and $s(\bar{V}) = \bar{V}^\top\!\1 \in \RR^r$ are the column sums of the mirror-stepped factors and the square root is elementwise. This is precisely the column-sum matching substep of Algorithm~\ref{alg:sinkhorn}, applied once with no iteration. 

\begin{algorithm}[t]
\caption{KL projection by cyclic Bregman projections (balanced retraction)}
\label{alg:sinkhorn}
\begin{algorithmic}[1]
\REQUIRE $\bar{U} > 0$, $\bar{V} > 0$, marginals $a, b$, tolerance $\varepsilon$, number of iterations $L_{\mathrm{sink}}$
\STATE $U \leftarrow \bar{U}$, $V \leftarrow \bar{V}$
\FOR{$\ell = 1, 2, \ldots, L_{\mathrm{sink}}$}
  \STATE $U \leftarrow \diag(a \oslash U\1)\,U$ \COMMENT{enforce $U\1 = a$}
  \STATE $V \leftarrow \diag(b \oslash V\1)\,V$ \COMMENT{enforce $V\1 = b$}
  \STATE $t \leftarrow \sqrt{s(V) \oslash s(U)}$; \quad $U \leftarrow U\diag(t)$; \quad $V \leftarrow V\diag(t^{-1})$ \COMMENT{enforce $s(U){=}s(V)$}
  \IF{max residual $\leq \varepsilon$}
    \STATE \textbf{break}
  \ENDIF
\ENDFOR
\RETURN $(U, V)$
\end{algorithmic}
\end{algorithm}

\begin{theorem}[Retraction validity]\label{thm:retraction}
Let the KL projection onto $\mathcal{M}$ be denoted by $\mathrm{KLProj}_{\mathcal{M}}(\bar{U}, \bar{V}) := \arg\min_{(U,V) \in \mathcal{M}} \KL(U \| \bar{U}) + \KL(V \| \bar{V})$ (computed by Algorithm~\ref{alg:sinkhorn} for $\Mbal$, or by~\eqref{eq:retr_ub} for $\Mub$).
Then the composite map $\Retr_{(U,V)}(t\Xi) := \mathrm{KLProj}_{\mathcal{M}}(\bar{U}, \bar{V})$, where $(\bar{U}, \bar{V})$ is given by~\eqref{eq:mirror}, is a smooth, first-order retraction: (i)~$\Retr_{(U,V)}(0) = (U,V)$; (ii)~$(d/dt)|_{t=0}\Retr_{(U,V)}(t\Xi) = \Xi$; (iii)~the image lies in $\mathcal{M}$ with strict positivity.
\end{theorem}

\subsection{Riemannian Gradient and Hessian}
To deploy Riemannian optimization algorithms, we require the Riemannian gradient and the Hessian-vector product. An advantage of our geometric framework is that these operators are agnostic to the OT objective: switching objectives (e.g., from linear OT to Gromov--Wasserstein) changes only the Euclidean derivatives.

\paragraph{Riemannian Gradient.} 
% Let $f(U,V)$ denote the objective function, and let $\nabla_E f = (\nabla_U f, \nabla_V f)$ be its standard Euclidean gradient. Under the Fisher-Rao metric, the ambient gradient is given by $W \odot \nabla_E f$, where $W = (U,V)$. The Riemannian gradient is then obtained by mapping this ambient gradient onto the tangent space using the orthogonal projector $\Pi$:
Let $\nabla_E f = (\nabla_U f, \nabla_V f)$ be the Euclidean gradient of the objective function $f(U,V)$. Under the Fisher-Rao metric the ambient gradient is $W \odot \nabla_E f$, where $W = (U,V)$. The Riemannian gradient is obtained by projecting this ambient gradient onto the tangent space: 
\begin{equation}
  \mathrm{grad}\,f = \Pi(U \odot \nabla_U f,\; V \odot \nabla_V f).
  \label{eq:rgrad}
\end{equation}
We note that the expression of $\mathrm{grad}\,f$ is identical for $\Mbal$ and $\Mub$ manifolds, with only the projector $\Pi$ being different. 
% In particular, this fundamental formula is identical for the balanced and unbalanced manifolds; only the choice of the projector $\Pi$ is different.
% dictates the specific geometry.

\paragraph{Riemannian Hessian-Vector Product.} 
Second-order optimization relies on the directional derivative of the gradient, formalized by the Riemannian Hessian applied to a tangent vector $\Xi = (\Xi_U, \Xi_V)$. 
The ambient Fisher-Rao Levi-Civita connection takes the form:
\begin{equation}
  \bar\nabla_\Xi \eta = (D\eta_U[\Xi] - \tfrac{1}{2}(\Xi_U {\odot} \eta_U) {\oslash} U,\; D\eta_V[\Xi] - \tfrac{1}{2}(\Xi_V {\odot} \eta_V) {\oslash} V). 
  \label{eq:connection}
\end{equation}
By applying this connection to the ambient gradient and projecting the result onto the tangent space, we obtain the Riemannian Hessian-vector product expression:
\begin{equation}
  \mathrm{Hess}\,f[\Xi] = \Pi\!\bigl(\tfrac{1}{2}\,\Xi \odot \nabla_E f + W \odot \nabla_E^2 f[\Xi]\bigr),
  \label{eq:hess_impl}
\end{equation}
where $\nabla_E^2 f[\Xi]$ is the Euclidean Hessian-vector product.
Evaluating this operator requires only one Euclidean Hessian-vector product, one element-wise scaling, and a single tangent projection. It  avoids the computationally tricky requirement of differentiating through $\Pi$ itself.
Formally,~\eqref{eq:hess_impl} instantiates the Gauss-equation expression $\mathrm{Hess}\,f[\Xi] = \Pi(\bar\nabla_\Xi\,\mathrm{grad}\,\bar f)$ for embedded Riemannian submanifolds~\cite[Sec.~5.3]{AbsilMahonySepulchre2008},\cite{Boumal2023}, applied to the ambient gradient $W\odot\nabla_E f$ along the Fisher--Rao Levi--Civita connection~\eqref{eq:connection}. The dependence of $\Pi$ on $(U,V)$ is absorbed into $\bar\nabla$ rather than appearing as a separate derivative term.

\begin{theorem}[Hessian symmetry]\label{thm:hessian}
For both $\mathcal{M}_{\mathrm{bal}}$ and $\mathcal{M}_{\mathrm{ub}}$, the Riemannian Hessian $\mathrm{Hess}\,f$ is a symmetric linear operator on the tangent space, satisfying 
$\langle \mathrm{Hess}\,f[\Xi], \eta \rangle_{(U,V)} = \langle \Xi, \mathrm{Hess}\,f[\eta] \rangle_{(U,V)}$ for all $\Xi, \eta \in T_{(U,V)}\mathcal{M}$. Furthermore, at any local minimum, $\mathrm{Hess}\,f$ is positive semidefinite.
\end{theorem}

The explicit, computationally tractable formulation of $\mathrm{Hess}\,f[\Xi]$~\eqref{eq:hess_impl} equips our framework with a second-order oracle. While existing low-rank OT methods are limited to first-order updates, the above developed Riemannian ingredients allows the use of curvature-aware solvers such as Riemannian trust-regions. Together with the retraction validity established in Theorem~\ref{thm:retraction}, the standard global and superlinear-local convergence guarantees for Riemannian conjugate gradients~\cite[Sec.~8.3]{AbsilMahonySepulchre2008} and trust-regions~\cite[Sec.~7.4]{AbsilMahonySepulchre2008} apply directly to our solvers. 
Please refer to Appendix~\ref{sec:app_convergence} for more details. 
%%Appendix%%
%The conditions on the vector transport are discussed in Appendix~\ref{sec:app_convergence}.

\subsection{Rank-Sufficiency Certificate}
\label{sec:certificate}
% A fundamental question in low-rank optimal transport is whether the chosen rank $r$ is large enough to capture the true geometry of the problem. When the underlying objective function $f$ is convex with respect to the full coupling matrix $\Gamma$, our Riemannian framework provides an easily computable certificate of global optimality. This allows us to verify whether a rank-$r$ critical point found on the manifold is, in fact, the unconstrained global minimum of the full-rank problem.

A fundamental question in the low-rank setting is whether the chosen rank $r$ is large enough to recover the full-rank optimum. When the objective $f$ is convex in the coupling $\Gamma$, our analysis below provides an easily computable certificate of global optimality. It verifies whether a rank-$r$ critical point found on the manifold is, in fact, the global minimum of the full-rank problem.

\begin{theorem}[Rank-sufficiency certificate]
\label{prop:certificate}
Let $f : \RR^{m \times n} \to \RR$ be convex and differentiable.
Consider the unconstrained {global problem} over all valid couplings:
\begin{equation}\label{eq:global_bal} %\label{eq:global_ub}
\begin{array}{ll}
  \textsc{Balanced:}\quad & \min_{\Gamma \in \Omega(a,b)} f(\Gamma), \\
  \textsc{Unbalanced:}\quad & \min\limits_{\Gamma \geq 0}\; F(\Gamma) \ (\coloneqq f(\Gamma) + \rho_1\,\KL(\Gamma\1 \| a) + \rho_2\,\KL(\Gamma^\top\!\1 \| b)),
\end{array}
\end{equation}
where $F(\cdot)$ denotes the unbalanced setting's optimization objective.
Let $(U^*, V^*)$ be a Riemannian critical point (i.e., $\mathrm{grad}\,f = 0$) of the corresponding rank-$r$ problem on $\mathcal{M}_{\mathrm{bal}}$ or $\mathcal{M}_{\mathrm{ub}}$, yielding the low-rank coupling $\Gamma_r = U^* \mathrm{diag}((U^*)^\top\1)^{-1} V^*$. We define the reduced gradient (or dual slack matrix) as: 
$$R \coloneqq \nabla_\Gamma f(\Gamma_r) - \alpha^* \mathbf{1}^\top - \mathbf{1} \beta^{*\top}, \qquad \delta_r \coloneqq \min\nolimits_{i,j} R_{ij},$$ where the dual variables $\alpha^* \in \mathbb{R}^m$ and $\beta^* \in \mathbb{R}^n$ are obtained as follows:

% and the \textbf{rank-$r$ problem}: the same objective restricted to couplings $\Gamma = U\,D\,V^\top$ with $\rank = r$ optimized over $\Mbal^{(r)}$ or $\Mub^{(r)}$ respectively.
% Let $(U^*, V^*)$ be a Riemannian critical point (i.e., $\mathrm{grad}\,f = 0$) of the rank-$r$ problem with coupling $\Gamma_r = U^* D(U^*) V^{*\top}$.
% Define the \emph{reduced gradient}
% \begin{equation}
%   R := \nabla_\Gamma f(\Gamma_r) - \alpha^* \1^\top - \1 \beta^{*\top},
%   \qquad
%   \delta_r := \min_{i,j}\, R_{ij},
%   \label{eq:reduced_grad}
% \end{equation}
% where $(\alpha^*, \beta^*)$ are dual variables defined as follows:
\begin{itemize}[nosep]
  \item \textsc{Balanced:} $\alpha^*, \beta^*$ are the Lagrange multipliers from the projector KKT system~\eqref{eq:KKT} at $(U^*, V^*)$.
  \item \textsc{Unbalanced:} $\alpha^*_i = -\rho_1\log(p_i/a_i)$, $\beta^*_j = -\rho_2\log(q_j/b_j)$, where $p = U^*\1$, $q = V^*\1$.
\end{itemize}
Then, the following properties hold:
\begin{enumerate}[nosep]
  \item[{(i)}] {Global optimality:}
  If $\delta_r \geq 0$, then $\Gamma_r$ is a {global minimizer} of the corresponding global problem (\ref{eq:global_bal}), even though it has rank $\leq r$.
  \item[{(ii)}] \textit{Stationarity identity (zero duality gap at $(\alpha^*,\beta^*)$).} $\ip{R}{\Gamma_r} = 0$. Equivalently, $\ip{\nabla_\Gamma f(\Gamma_r)}{\Gamma_r} = a^\top\alpha^* + b^\top\beta^*$ in the balanced case, and $\ip{\nabla_\Gamma F(\Gamma_r)}{\Gamma_r} = 0$ in the unbalanced case.
  \item[{(iii)}] \textit{Sub-optimality:} If $\delta_r < 0$, then $(\alpha^*,\beta^*)$ are not dual-feasible for the global problem~\eqref{eq:global_bal}, so $\Gamma_r$ is not a global optimum. The index $(i^*,j^*) \in \arg\min_{i,j} R_{ij}$ identifies a direction along which the full-rank objective can be decreased.%, in the spirit of Dantzig--Wolfe column generation~\cite{DantzigWolfe1960}.
\end{enumerate}
\end{theorem}

% \begin{remark}[Open parametrization and boundary effects]
% \label{rem:interior_certificate}
% The Riemannian iterate satisfies $\Gamma_{r,ij} > 0$ for all $(i,j)$, since the factors $(U^*,V^*)$ live in the open positive orthant. Combined with Part~(ii) ($\ip{R}{\Gamma_r} = 0$), the binary condition $\delta_r \ge 0$ thus implies $R \equiv 0$ entrywise, i.e., $C_{ij} = \alpha^*_i + \beta^*_j$ for all $(i,j)$ (a rank-$1$ structural condition on $C$). For generic cost matrices this fails at any finite iterate, so the strict inequality $\delta_r \ge 0$ is rarely satisfied exactly: $\delta_r \ge 0$ characterizes a \emph{limit} of the Riemannian trajectory as it approaches the boundary of the closed nonneg orthant where the full-rank optimum sits. Corollary~\ref{cor:gap_bound} below replaces the binary test with a finite-iterate gap bound that decays continuously as $\delta_r \to 0^-$.
Since the factors $(U^*,V^*)$ lie in the open positive orthant, every entry of the Riemannian iterate is strictly positive, $\Gamma_{r,ij} > 0$. Together with $\ip{R}{\Gamma_r} = 0$ (point (ii) in Theorem~\ref{prop:certificate}), this forces $R \equiv 0$ entrywise whenever $\delta_r \ge 0$. For linear OT ($\nabla_\Gamma f = C$), $R \equiv 0$ implies $C = \alpha^* \mathbf{1}^\top + \mathbf{1} \beta^{*\top}$. This in turn implies that $\rank(C)\leq 2$, a condition cost matrices need not satisfy. The exact test $\delta_r \ge 0$ therefore holds only in the limit, as the iterates approach the full-rank optimum, which is typically sparse and so sits on the boundary of the nonnegative orthant, where some entries $\Gamma_{ij}$ vanish. Corollary~\ref{cor:gap_bound} replaces this exact test with a gap bound that holds at every finite iterate and shrinks continuously as $\delta_r \to 0^-$.
% \end{remark}

\begin{corollary}[Quantitative suboptimality bound]
\label{cor:gap_bound}
Under the hypotheses of Theorem~\ref{prop:certificate}, let $\Gamma^*$ denote a global minimizer of the corresponding problem in~\eqref{eq:global_bal} and let $(-\delta_r)_+ := \max(0,-\delta_r)$. Then the suboptimality gap of $\Gamma_r$ is at most $(-\delta_r)_+$ times the transported mass:
\begin{itemize}[nosep]
  \item \emph{Balanced:} $\;f(\Gamma_r) - f(\Gamma^*) \;\le\; (-\delta_r)_+ \cdot \1^\top a.$
  \item \emph{Unbalanced:} $\;F(\Gamma_r) - F(\Gamma^*) \;\le\; (-\delta_r)_+ \cdot \1^\top\Gamma^*\1.$
\end{itemize}
The gap shrinks to zero as $\delta_r \to 0^-$, recovering the exact test of Theorem~\ref{prop:certificate}.
% In particular, $\delta_r = 0$ recovers the exact global-optimality criterion of Part~(i); for $\delta_r < 0$, the suboptimality gap is bounded linearly by $|\delta_r|$ times the total transported mass.
\end{corollary}

The result follows from the KKT conditions for convex optimization over the non-negative orthant~\cite{BoydVandenberghe2004}: at Riemannian stationarity, the marginal-constraint multipliers $(\alpha^*,\beta^*)$ from the projector are exactly those for which the duality gap at $\Gamma_r$ vanishes (Part~(ii)). The binary criterion $\delta_r \ge 0$ then tests whether these same multipliers are additionally dual-feasible for the full-rank problem.

In practice, the certificate is a post-hoc diagnostic: after convergence at rank $r$, compute $R$ and inspect $\delta_r$. If $\delta_r \ge 0$, the rank-$r$ solution is globally optimal. If $\delta_r < 0$, the most negative entry identifies a rank-increase direction. 
The multipliers $\alpha^*,\beta^*$ come for free: in closed form in the unbalanced setting, and as a byproduct of the projector's KKT solve in the balanced one.
%For the unbalanced setting, $\alpha^*,\beta^*$ are available in closed form, so the check requires no linear solves. For the balanced setting, $\alpha^*,\beta^*$ are already produced as a byproduct of the projector's KKT solve at convergence.

%\paragraph{Specialization to linear OT.}
%For $f(\Gamma) = \ip{C}{\Gamma}$, the reduced gradient $R = C - \alpha^*\1^\top - \1\beta^{*\top}$ is the classical LP reduced cost.
%The certificate $\delta_r \geq 0$ is equivalent to LP dual feasibility; by strong duality, $\ip{C}{\Gamma_r} = a^\top\alpha^* + b^\top\beta^*$ (zero duality gap).
%Moreover, the Kantorovich LP has an optimal vertex of rank $\leq m{+}n{-}1$, so the certificate is guaranteed to hold once $r \geq m{+}n{-}1$.
%For $\delta_r < 0$, the most negative entry $(i^*, j^*)$ identifies a rank-increase direction, recovering classical Dantzig--Wolfe column generation~\cite{DantzigWolfe1960}.
%
%\paragraph{Practical significance.}
%The certificate is a \emph{post-hoc} check: after convergence at rank $r$, compute $R$ (cost $O(mn)$ for dense $\nabla_\Gamma f$, or $O((m{+}n)r)$ if $\nabla_\Gamma f$ admits a factored evaluation) and inspect $\delta_r$.
%If $\delta_r \geq 0$, the rank-$r$ solution is provably optimal.
%If $\delta_r < 0$, the user can increase rank and re-optimize.
%No existing low-rank OT baseline provides such a certificate.

% ======================================================================
%  4. COST FUNCTIONS AND COMPLEXITY
% ======================================================================
\section{Cost Functions and Complexity}
\label{sec:costs}

A defining advantage of our Riemannian framework is its modularity: the manifold machinery (metric, projector, retraction, and Hessian-vector connection) depends strictly on the geometry of the factors $(U,V)$ and the chosen marginal constraints. It is entirely decoupled from the optimal transport cost function. Transitioning between different OT formulations requires modifying only the Euclidean gradient and the Euclidean Hessian-vector product. Below, we instantiate this for four standard cost functions and analyze the resulting computational complexity against existing baselines.

% The manifold machinery (metric, projector, retraction, Hessian formula) depends only on the current factors $(U,V)$, not on the cost function.
% Switching cost functions requires changing only the Euclidean gradient and Hessian-vector product.
% We describe four cost functions and then compare the per-iteration complexity against baselines.

\paragraph{Balanced linear OT.}
On the balanced manifold $\mathcal{M}_{\mathrm{bal}}$, the standard linear Kantorovich transport cost evaluated under our low-rank framework is $f(U,V) = \sum_{k=1}^{r} \frac{u_k^\top C\,v_k}{\1^\top u_k}$, 
% \begin{equation}
%   f(U,V) = \sum_{k=1}^{r} \frac{u_k^\top C\,v_k}{\1^\top u_k},
%   \label{eq:transportcost}
% \end{equation}
where $C \in \mathbb{R}^{m \times n}$ is the ground cost matrix, and $u_k, v_k$ denote the $k$-th columns of $U$ and $V$, respectively. It may be observed that the above objective is same as  $\langle C, \Gamma \rangle$, where $\Gamma$ is given by (\ref{eqn:gamma}) as $g_k = \mathbf{1}^\top u_k$.

\paragraph{Balanced Gromov-Wasserstein (GW).} 
The GW problem~\cite{Memoli2011,PeyreBook2019} compares two metric measure spaces by seeking a coupling $\Gamma$ that preserves pairwise intra-domain distances. Let $C^X \in \mathbb{R}^{m \times m}$ and $C^Y \in \mathbb{R}^{n \times n}$ denote symmetric intra-domain distance matrices. When using a squared loss, minimizing the GW discrepancy $\mathrm{GW}(\Gamma)\coloneqq\sum_{i,j,k,l} (C^X_{ik} - C^Y_{jl})^2 \Gamma_{ij}\Gamma_{kl}$ over the prescribed marginals is algebraically equivalent to maximizing the following cross-domain interaction in low-rank setting with $(U,V)\in\Mbal$: 
\begin{equation}
  f_{\mathrm{GW}}(U,V) = \mathrm{tr}(V \mathrm{diag}(g)^{-1} U^\top C^X U \mathrm{diag}(g)^{-1} V^\top C^Y)
  \label{eq:gw_factored}
\end{equation}
where $g=U^\top\!\1$. 
%When the domains are equipped with general distance matrices, this costs $O(m^2 r + n^2 r)$. 
However, when employing squared-Euclidean distances for point clouds $X \in \mathbb{R}^{m \times d}$ and $Y \in \mathbb{R}^{n \times d}$ (i.e., $C^X_{ik} = \|x_i - x_k\|^2$), we can compute the objective, its gradients, and the Hessian-vector product via matrix-vector multiplications directly with the data matrices $X, Y$ when $d\ll m,n$. %. This results in $O((m+n)dr)$ per iteration computational complexity, which is useful 

% computable in $O(r^2 m + r^2 n + r^3)$ or $O(mdr)$ for squared-Euclidean distances (without forming the $m{\times}m$ matrix $C^X$). \alertbyPJ{Does the last statement imply that we use $XX^\top$ and $YY^\top$ without explicitly computing them?}

\paragraph{Unbalanced linear OT.} In this setting, strict marginal constraints are not required. Instead, deviations from source and target marginals are penalized using the generalized Kullback-Leibler divergence. This problem may be posed over the $\Mub$ manifold as follows: 
\begin{equation}
  \min_{(U,V) \in \Mub}\; \sum_{k=1}^r \frac{u_k^\top C\,v_k}{\1^\top u_k} + \;\rho_1\,\KL(U\1 \| a) + \rho_2\,\KL(V\1 \| b),
  \label{eq:uot_cost}
\end{equation}
where $\KL(p\|q) = \sum_i [p_i \log(p_i/q_i) - p_i + q_i]$ is the generalized KL divergence and $\rho_1, \rho_2 > 0$ control marginal fit. 

\paragraph{Unbalanced GW and Fused GW.}
Unbalanced GW (UGW) and Fused GW (FGW)~\cite{Vayer2019} combine the GW and linear-OT costs with the KL marginal penalties:
\begin{equation}\label{eq:ufgw}
\min_{(U,V)\in\Mub}\ \alpha\,\langle C_{\mathrm{inter}},\Gamma\rangle + (1-\alpha)\,\mathrm{GW}(\Gamma) + \tau(\KL(U\1\|a) + \KL(V\1\|b)),
\end{equation}
where $\Gamma = U\diag(g)^{-1}V^\top$, $\langle C_{\mathrm{inter}},\Gamma\rangle$ is the linear OT term capturing feature information,  and $\mathrm{GW}(\Gamma)$ is the squared-loss GW cost term capturing structural information (geometry). The weight $\alpha\in[0,1]$ trades off the feature and geometry terms ($\alpha=1$ is UOT while $\alpha=0$ is unbalanced GW), while $\tau>0$ controls the marginal relaxation through the KL penalties ($\rho_1=\rho_2=\tau$). These composite objectives are seamlessly handled by the proposed Riemannian framework.
% Unbalanced GW and Fused GW (FGW)~\cite{Vayer2019} extend the formulations above by combining the standard GW and linear OT costs with the KL marginal penalties of unbalanced OT. These composite objectives can be seamlessly handled via the proposed Riemannian framework.

% The unbalanced GW cost adds marginal-dependent quadratic terms $p^\top (C^X)^{\odot 2} p + q^\top (C^Y)^{\odot 2} q$ that are constant in the balanced case but contribute to the gradient here.
% The \emph{Fused Gromov--Wasserstein} (FGW)~\cite{Vayer2019} combines inter-domain features and intra-domain geometry:
% \begin{equation}
%   \min_{(U,V) \in \Mub}\; \alpha\,\ip{C_{\mathrm{inter}}}{\Gamma} + (1{-}\alpha)\,\mathrm{GW}(\Gamma) + \rho_1\,\KL(U\1 \| a) + \rho_2\,\KL(V\1 \| b),
%   \label{eq:ufgw}
% \end{equation}
% where $\alpha \in [0,1]$ interpolates between pure OT ($\alpha{=}1$) and pure GW ($\alpha{=}0$).
% The Euclidean gradient is $\alpha\,\nabla_E f_{\mathrm{OT}} + (1{-}\alpha)\,\nabla_E f_{\mathrm{GW}} + \nabla_E f_{\mathrm{penalty}}$; the Riemannian conversion uses~\eqref{eq:rgrad} unchanged.

\paragraph{Computational complexity.}
Let $N = m + n$. Following standard practice in low-rank OT~\cite{ScetbonCuturiPeyre2021,ScetbonPeyreCuturi2022}, the ground cost is assumed to admit a factored (low-rank) form, under which cost evaluations scale as $O(rN)$. The per-iteration computational complexity of our first- and second-order Riemannian algorithms is $O((k_{\mathrm{CG}} + L_{\mathrm{sink}})\,rN)$ for the balanced setting and strictly $O(rN)$ for the unbalanced setting. Asymptotically, these costs exactly match existing state-of-the-art first-order baselines, including LOT~\cite{ScetbonCuturiPeyre2021}, LR-GW~\cite{ScetbonPeyreCuturi2022}, and UB-LOT~\cite{ScetbonSejourne2023}, which require $O(L_{\mathrm{in}}\,rN)$ per mirror descent iteration to perform $L_{\mathrm{in}}$ inner sweeps of Dykstra or Bregman updates. Crucially, our framework achieves this while providing a full second-order oracle (Hessian-vector products). Furthermore, our unbalanced geometry structurally guarantees fast iterations. In UB-LOT~\cite{ScetbonSejourne2023}, the mirror descent subproblems couple row-wise marginal penalties with exact column-wise latent consistency constraints, forcing the use of iterative Dykstra sweeps. In contrast, our Riemannian approach isolates the marginal penalties strictly within the Euclidean gradient. The geometric operations are thus decoupled from the penalties, causing the inner loops to vanish entirely and dropping the per-iteration complexity to the base evaluation cost of $O(rN)$. All methods share the same $O(rN)$ memory footprint.

\section{Numerical Experiments}
\label{sec:experiments}
In this section, we benchmark our approach against state-of-the-art in various low-rank OT settings. We begin with balanced experiments in linear OT and GW setting (Sections~\ref{sec:exp_lot_bench}--\ref{sec:app_anisotropic_gw}).  We then turn to the unbalanced setting (Sections~\ref{sec:exp_outlier}--\ref{sec:exp_ub_gw_fgw}), where the closed-form projector and retraction on $\Mub$ eliminate all inner iterations, and the speed advantage of our framework is most pronounced.

\textbf{Baselines.} We compare our Riemannian solvers against existing low-rank OT solvers such as: (a) LOT~\cite{ScetbonCuturiPeyre2021} in linear OT setting, (b) LR-GW~\cite{ScetbonPeyreCuturi2022} in Gromov-Wasserstein (GW) OT setting, and (c) UB-LOT~\cite{ScetbonSejourne2023} in unbalanced OT setting. We also include comparisons with FactoredOT~\cite{Forrow2019} and FRLC~\cite{Halmos2024} in the appendix. 
The hyper-parameter details are as follows.

\textbf{Setup.} Our algorithms are implemented in PyManopt~\cite{TownsendPyManopt2016}. 
The steepest descent (SD) and conjugate gradients (CG) use Armijo line search, while our trust-regions (TR) uses $\ge 3$ inner truncated-CG steps. All remaining TR parameters (step-acceptance threshold $\rho'{=}0.1$, inner tCG tolerances $\kappa{=}0.1$ and $\theta{=}1$, initial radius $\Delta_0{=}\Delta_{\mathrm{bar}}/8$ with $\Delta_{\mathrm{bar}}{=}\sqrt{\dim\mathcal{M}}$) are kept at their PyManopt defaults. 
Our solvers require neither entropic regularization nor a step-size parameter. In contrast, LOT and LR-GW require an initial step-size~$\gamma_0$ (set as $10$) and entropic regularization~$\varepsilon$. While discussing the results, we report their unregularized cost. UB-LOT requires a step-size  parameter~$\gamma \in \{10, 70, 200\}$ (denoted UB-LOT$_\gamma$ in tables). Additional hyper-parameter tuning details are discussed in the appendix.

All the experiments are performed on an Intel i7-12700H with 64\,GB RAM and the code is available at \url{https://github.com/pratikjawanpuria/RiemannianLowRankOT}. Every method terminates when the relative  change in cost falls below $10^{-6}$ over a window of $5$ iterations, subject to a hard cap of $1000$ outer iterations in the balanced setting and $500$ in the unbalanced setting. In unbalanced setting, we also report  the total transported mass $\1^\top \Gamma \1$ (``Mass'') and the orphan fraction (``Orph''), i.e., the mass assigned to target clusters absent from the source. 

For GW (and FGW) experiments, the algorithms maximize the cross-domain interaction $f_{\mathrm{GW}}$ in~\eqref{eq:gw_factored}. Thus, we report the full squared-loss GW discrepancy, which decomposes as a marginal-dependent constant plus $-2\,f_{\mathrm{GW}}$, so that smaller values are better. The reported values may therefore be negative whenever $2\,f_{\mathrm{GW}}$ exceeds the marginal-dependent constant.

\subsection{Balanced Linear OT}
\label{sec:exp_lot_bench}
We follow the experimental setup of ~\cite{ScetbonCuturiPeyre2021}. 
%We replicate the benchmark from Figure~3 of~\cite{ScetbonCuturiPeyre2021}.
The source is a mixture of $3$ Gaussians in $\RR^2$ and the target is a mixture of $2$ Gaussians, with $n{=}m{=}10\,000$ uniform-weight points (exact OT cost $f^*{=}0.287$). 
We compare Riemannian (SD, CG, TR) and LOT with $\varepsilon \in \{0, 5{\times}10^{-3}, 0.05\}$ at ranks $r \in \{10, 50, 100, 500\}$. The Sinkhorn algorithm \cite{Cuturi2013} has $f/f^*{=}1.016$ and it serves as a full-rank reference. 
%For our approach, the CG-based projection (Remark~\ref{rem:cg}) is used for $r \ge 50$.  
Table~\ref{tab:lot_benchmark} reports results at $r=\{10,500\}$. Please refer to Appendix~\ref{sec:app_lot_full} for results at other ranks. %%Appendix%%
\begin{table}[t]
\caption{Cost ratio $f/f^*$ (lower is better), iteration count, and wall-clock time for balanced OT on the Gaussian-mixture benchmark of~\cite{ScetbonCuturiPeyre2021} ($n{=}m{=}10\text{k}$, $\RR^2$, exact cost $f^*{=}0.287$; Sinkhorn reference: $f/f^*{=}1.016$).  Riem~CG matches the exact LP at $r{=}500$; LOT with $\varepsilon{>}0$ suffers entropic bias.  \textbf{Bold}: best ratio per rank.}
\label{tab:lot_benchmark}
\small
\centering
\begin{minipage}[t]{0.48\textwidth}
\centering
\setlength{\tabcolsep}{2pt}
\begin{tabular}{@{}l ccc@{}}
\toprule
\multicolumn{4}{@{}l}{\textbf{Rank $r{=}10$}} \\
\midrule
Method & $f/f^*$ & Iters & Time (s) \\
\midrule
Riem~SD  & $1.268$ & $23$   & $93$ \\
Riem~CG  & $1.253$ & $8$    & $73$ \\
Riem~TR  & $1.280$ & $11$   & $67$ \\
LOT $\varepsilon{=}0$       & $1.261$ & $213$  & $52$ \\
LOT $\varepsilon{=}5{\times}10^{-3}$ & $\mathbf{1.247}$ & $1001$ & $246$ \\
LOT $\varepsilon{=}0.05$    & $2.009$ & $1001$ & $37$ \\
\bottomrule
\end{tabular}
\end{minipage}\hfill
\begin{minipage}[t]{0.48\textwidth}
\centering
\setlength{\tabcolsep}{2pt}
\begin{tabular}{@{}l ccc@{}}
\toprule
\multicolumn{4}{@{}l}{\textbf{Rank $r{=}500$}} \\
\midrule
Method & $f/f^*$ & Iters & Time (s) \\
\midrule
Riem~SD  & $1.041$ & $23$   & $1035$ \\
Riem~CG  & $\mathbf{1.000}$ & $10$   & $1764$ \\
Riem~TR  & $1.083$ & $16$   & $1008$ \\
LOT $\varepsilon{=}0$       & $1.063$ & $235$  & $1454$ \\
LOT $\varepsilon{=}5{\times}10^{-3}$ & $1.118$ & $1001$ & $7289$ \\
LOT $\varepsilon{=}0.05$    & $1.936$ & $555$  & $1824$ \\
\bottomrule
\end{tabular}
\end{minipage}
\end{table}
We note that while Riemannian solvers require far fewer outer iterations, LOT's wall-clock time can be comparable due to the cheaper per-iteration cost of mirror descent. We also observe that the proposed solvers achieve better objective values at $r=500$ (Figure~\ref{fig:conv_balanced}a). 
%The most striking result is at $r{=}500$: Riem~CG recovers the exact LP solution ($f/f^*{=}1.000$) in only $10$ outer iterations, while LOT ($\varepsilon{=}0$) plateaus at $1.063$ after $235$ iterations .
% This gap widens with rank (at $r{=}10$ the methods are nearly tied on cost), suggesting that the Riemannian gradient provides increasingly effective descent directions as the manifold dimension grows.
We also observe that the entropic regularization biases the solution as LOT's objective $\varepsilon{=}0.05$ is high.  
% Entropic regularization proves harmful: LOT $\varepsilon{=}0.05$ is stuck at ratio ${\sim}2.0$ regardless of rank, and even the mild $\varepsilon{=}5{\times}10^{-3}$ degrades quality at high rank ($1.118$ vs.\ $1.063$), confirming that our regularization-free formulation avoids an inherent bias.
% On the other hand, LOT's per-iteration cost is lower, so its wall-clock time is competitive at small ranks (e.g., $52$\,s vs.\ $67$\,s at $r{=}10$).

% the Riemannian advantage is primarily in solution quality and iteration count. 
% rather than raw speed for balanced OT.

\subsection{Balanced GW}
\label{sec:app_anisotropic_gw}

We next test on a large-scale, ill-conditioned GW problem with the source having $m{=}10{,}000$ points sampled from $\mathcal{N}(0, I_{50})$ in $\RR^{50}$ and the target having $n{=}20{,}000$ points sampled from $\mathcal{N}(0, \Sigma)$ in $\RR^{100}$ with $\kappa(\Sigma){=}10^6$.
All methods use factored (matrix-free) GW cost evaluation.
Table~\ref{tab:anisotropic_gw} reports results at $r{=}5$ and $r{=}20$, with  results at $r{=}10$ in Appendix~\ref{sec:app_anisotropic_full}. %%Appendix%%

\begin{table}[t]
\caption{GW cost, iteration count, and wall-clock time for balanced GW on anisotropic Gaussians ($m{=}10\text{k}$ in $\RR^{50}$, $n{=}20\text{k}$ in $\RR^{100}$, $\kappa(\Sigma){=}10^6$; factored evaluation).  Riemannian solvers converge in under $15$ iterations; LR-GW exhausts $1001$ but achieves slightly lower cost at $r{=}20$.  Regularization $\varepsilon{=}0.01$ collapses to the marginal-only term.  \textbf{Bold}: best cost per rank.}
\label{tab:anisotropic_gw}
\small
\centering
\begin{minipage}[t]{0.48\textwidth}
\centering
\setlength{\tabcolsep}{2pt}
\begin{tabular}{@{}l ccc@{}}
\toprule
\multicolumn{4}{@{}l}{\textbf{Rank $r{=}5$}} \\
\midrule
Method & Cost & Iters & Time (s) \\
\midrule
Riem~SD  & $7519$ & $9$    & $76$ \\
Riem~CG  & $\mathbf{7513}$ & $14$   & $16$ \\
Riem~TR  & $7530$ & $182$  & $12$ \\
LR-GW $\varepsilon{=}0$      & $\mathbf{7513}$ & $1001$ & $256$ \\
LR-GW $\varepsilon{=}.005$   & $7573$ & $409$  & $24$ \\
LR-GW $\varepsilon{=}.01$    & $7614$ & $26$   & $5$ \\
\bottomrule
\end{tabular}
\end{minipage}\hfill
\begin{minipage}[t]{0.48\textwidth}
\centering
\setlength{\tabcolsep}{2pt}
\begin{tabular}{@{}l ccc@{}}
\toprule
\multicolumn{4}{@{}l}{\textbf{Rank $r{=}20$}} \\
\midrule
Method & Cost & Iters & Time (s) \\
\midrule
Riem~SD  & $7509$ & $7$    & $48$ \\
Riem~CG  & $7508$ & $9$    & $56$ \\
Riem~TR  & $7515$ & $160$  & $42$ \\
LR-GW $\varepsilon{=}0$      & $\mathbf{7506}$ & $1001$ & $314$ \\
LR-GW $\varepsilon{=}.005$   & $7576$ & $300$  & $36$ \\
LR-GW $\varepsilon{=}.01$    & $7614$ & $34$   & $5$ \\
\bottomrule
\end{tabular}
\end{minipage}
\end{table}

Unlike the convex linear-OT setting, the GW objective is non-convex and poses a more challenging optimization landscape. We observe that the proposed Riemannian approaches closely match the LR-GW ($\varepsilon{=}0$) objective but at much lower wall-clock time and iterations. 
% Interestingly, Riem~CG matches LR-GW ($\varepsilon{=}0$) exactly at $r{=}5$ (cost $7513$), while at $r{=}20$ LR-GW reaches a marginally lower cost ($7506$ vs.\ $7508$) but requires all $1001$ allocated iterations ($314$\,s). On the other hand, CG converges to within $0.03\%$ in $9$ iterations ($56$\,s). SD and CG consistently converge in under $15$ outer iterations across all tested ranks.
We again observe that the entropic regularization biases the solution of LR-GW. In particular, at $\varepsilon{=}0.01$, we analyzed that the LR-GW's solution degenerates to the marginal-only term at every rank, meaning the coupling ignores pairwise geometry entirely. %This is absent from the regularization-free Riemannian approach.

Figure~\ref{fig:conv_balanced}(b) shows convergence at $r{=}20$, where Riemannian solver converge to good values faster. While Riem SD and Riem CG are the fastest, Riem TR starts converging efficiently after finding the right basin of attraction. We further observe in Appendix~\ref{sec:exp_verylargescale} that large scale experiments on both balanced OT and GW have similar qualitative trends discussed above.
%retain these trends up to $n{=}50{,}000$. %%Appendix%%

\begin{figure}[t]
\centering
\begin{subfigure}[b]{0.32\textwidth}
  \centering
  \includegraphics[width=\linewidth]{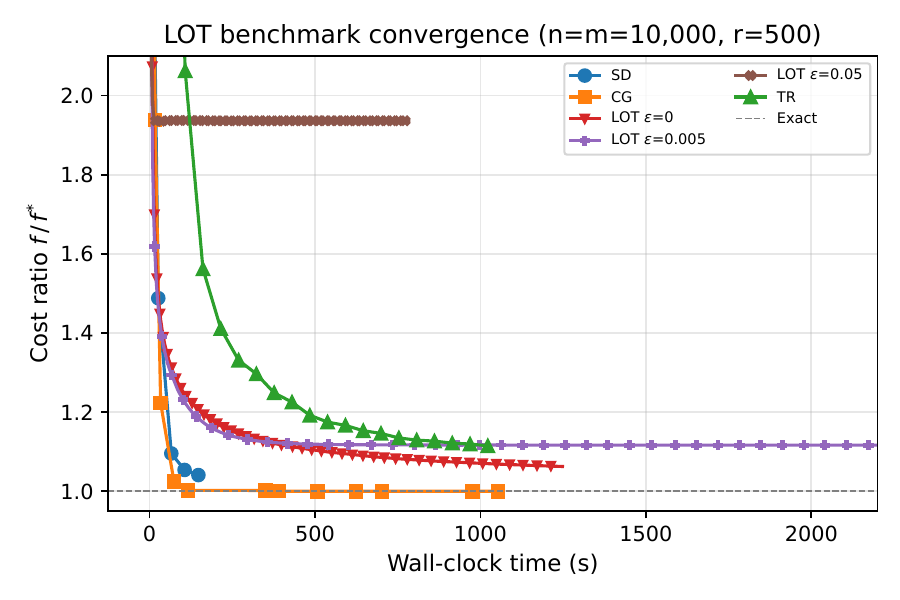}
  \caption{Linear OT ($r{=}500$)}
  \label{fig:conv_lot_bench}
\end{subfigure}
\hfill
\begin{subfigure}[b]{0.32\textwidth}
  \centering
  \includegraphics[width=\linewidth]{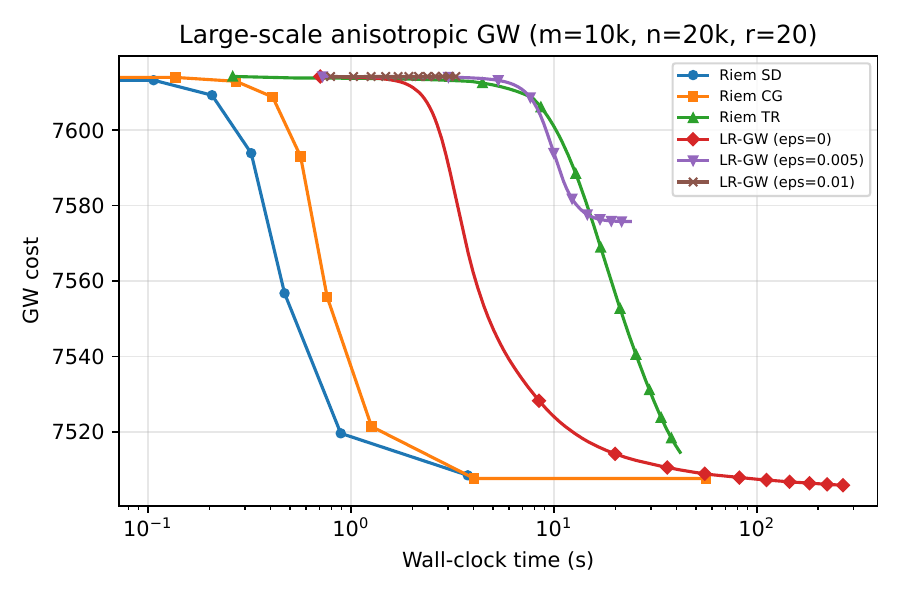}
  \caption{GW ($r{=}20$)}
  \label{fig:conv_anisotropic_gw}
\end{subfigure}
\hfill
\begin{subfigure}[b]{0.32\textwidth}
  \centering
  \includegraphics[width=\linewidth]{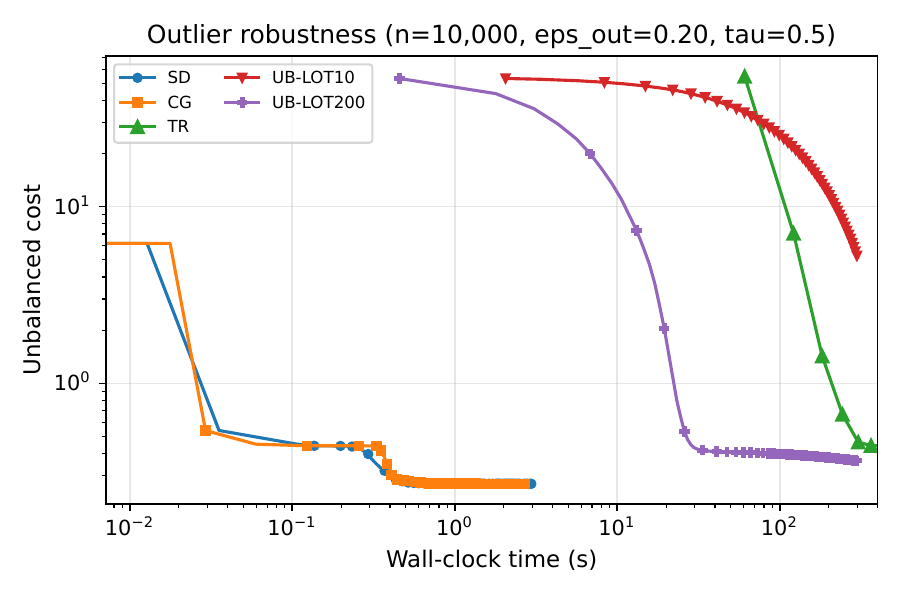}
  \caption{Unbal. OT ($\varepsilon_{\mathrm{out}}=0.20$)}
  \label{fig:conv_outlier}
\end{subfigure}
\caption{Cost vs.\ wall-clock time.  (a)~Balanced OT, $r{=}500$: CG reaches the exact LP cost.  (b)~Balanced GW, $r{=}20$: Riemannian methods converge in seconds; LR-GW decreases slowly over $1001$ iterations.  (c)~Unbalanced OT, $20\%$ outliers: SD/CG converge in seconds; UB-LOT times out at $300$\,s.}
\label{fig:conv_balanced}
\end{figure}

\subsection{Unbalanced OT: Outlier Robustness}
\label{sec:exp_outlier}

In this setting, source and target are $d{=}3$ Gaussian clusters contaminated with a fraction $\varepsilon_{\mathrm{out}} \in \{0.05, 0.10, 0.20\}$ of points replaced by uniform noise. We consider sizes $n \in \{5000, 10000\}$, rank $r{=}5$, and sweep $\tau \in \{0.5, 1, 5\}$. Table~\ref{tab:outlier} reports two representative configurations. The full $18$-configuration grid is in Appendix~\ref{sec:app_outlier_full}.%%Appendix%%

\begin{table}[t]
\caption{Unbalanced OT cost and transported mass on Gaussian clusters contaminated with outlier noise in two representative configurations. Riemannian solvers find $40$--$60\%$ lower cost than UB-LOT without parameter tuning; UB-LOT$_{10}$ diverges at high contamination.  \textbf{Bold}: lowest cost.}
\label{tab:outlier}
\small
\centering
\begin{minipage}[t]{0.48\textwidth}
\centering
\setlength{\tabcolsep}{2pt}
\begin{tabular}{@{}l cc@{}}
\toprule
\multicolumn{3}{@{}l}{\textbf{$n{=}5000$, $\varepsilon_{\mathrm{out}}{=}0.05$}} \\
\midrule
Method & Cost & Mass \\
\midrule
Riem~SD          & $\mathbf{0.201}$ & $0.80$ \\
Riem~CG          & $0.204$ & $0.79$ \\
Riem~TR          & $0.204$ & $0.79$ \\
UB-LOT$_{10}$    & $0.517$ & $0.97$ \\
UB-LOT$_{200}$   & $0.416$ & $0.84$ \\
\bottomrule
\end{tabular}
\end{minipage}\hfill
\begin{minipage}[t]{0.48\textwidth}
\centering
\setlength{\tabcolsep}{2pt}
\begin{tabular}{@{}l cc@{}}
\toprule
\multicolumn{3}{@{}l}{\textbf{$n{=}10000$, $\varepsilon_{\mathrm{out}}{=}0.20$}} \\
\midrule
Method & Cost & Mass \\
\midrule
Riem~SD          & $0.268$ & $0.73$ \\
Riem~CG          & $\mathbf{0.267}$ & $0.73$ \\
Riem~TR          & $\mathbf{0.267}$ & $0.73$ \\
UB-LOT$_{10}$    & $7.620$ & $0.83$ \\
UB-LOT$_{200}$   & $0.488$ & $0.75$ \\
\bottomrule
\end{tabular}
\end{minipage}
\end{table}

Our Riemannian solver achieves the lowest cost in every one of the $18$ configurations, typically $40$--$60\%$ below the best UB-LOT variant. 
This is because UB-LOT's parameter $\gamma$ simultaneously controls step size and marginal relaxation, making it impossible to tune one independently of the other. Small $\gamma$ ($=10$) under-relaxes the marginals, retaining too much outlier mass. Large $\gamma$ ($=200$) converges too slowly (Figure~\ref{fig:conv_balanced}c).
At $n{=}10{,}000$, $\varepsilon_{\mathrm{out}}{=}0.20$, UB-LOT$_{10}$ diverges entirely (cost $7.62$, nearly $28{\times}$ the Riemannian cost).
On the other hand, SD and CG require no parameter tuning and converge in $1$--$10$\,s. TR is also parameter-free but can be slower on individual configurations. It should be noted that this speed advantage is a direct consequence of the closed-form projector and retraction on $\Mub$ (Section~\ref{sec:manifold}): each Riemannian iteration costs $O(rN)$ with no inner loops, whereas UB-LOT runs Dykstra sweeps at every outer step.

\paragraph{Partial transport ($m$ up to $15{,}000$).} We next map $3$ Gaussian clusters to $2$ in $\RR^5$ at sizes up to $m{=}15{,}000$. Our Riemannian solvers attain the lowest objective in $5$ of $9$ configurations and the fastest wall-clock time in all $9$, finishing in seconds where UB-LOT needs minutes. 
%UB-LOT is competitive on objective primarily under strong marginal relaxation ($\tau{=}0.5$), where its inner Dykstra sweeps help discard mass; for $\tau{\ge}1$, the high-mass-retention regime common in practice, our solvers match or beat it. 
Full results are in Appendix~\ref{sec:app_partial_full}. %%Appendix%%

% \paragraph{Partial transport ($m$ up to $15{,}000$).} We also evaluate partial transport (source has $3$ Gaussian clusters in $\RR^5$, target has $2$ Gaussian clusters in $\RR^5$, with sizes up to $m{=}15{,}000$). Our Riemannian solvers obtain better objective in $5$ out of $9$ configurations and better wall-clock time in all $9$ configurations, finishing in seconds where UB-LOT requires minutes.
% UB-LOT achieves lower cost at the lowest penalty ($\tau{=}0.5$, where marginal relaxation matters most), but times out at larger sizes.
% Full results are in Appendix~\ref{sec:app_partial_full}.

\subsection{Unbalanced FGW}
\label{sec:exp_ub_gw_fgw}

\paragraph{Fused GW with partial overlap.}
% We transport $3$ Gaussian clusters ($m{=}300$, feature dimension $d_{\mathrm{feat}}{=}5$, $d_{\mathrm{geom}}{=}3$) to $2$ of $3$ clusters ($n{=}200$), rank $r{=}5$, sweeping the OT/GW trade-off $\alpha \in \{0.25, 0.5\}$ at $\tau{=}5$.
We transport a source of $3$ Gaussian clusters ($m{=}300$ points, with $d_{\mathrm{feat}}{=}5$ feature dimensions and $d_{\mathrm{geom}}{=}3$ geometry dimensions) to a target containing only $2$ of these $3$ clusters ($n{=}200$). The unmatched source cluster has no counterpart and should be left untransported (reflected in the orphan fraction). We use rank $r{=}5$ and sweep the feature/geometry trade-off $\alpha \in \{0.25, 0.5\}$ at marginal penalty $\tau{=}5$~\eqref{eq:ufgw}. 
While Table~\ref{tab:fgw_main} reports selected results, more results across all $(\alpha, \tau)$ are discussed in Appendix~\ref{sec:app_fgw_full}. %%Appendix%%

\begin{table}[t]
\caption{FGW cost, transported mass, orphan fraction, and wall-clock time for the partial-overlap experiment ($3{\to}2$ Gaussian clusters, $m{=}300$, $n{=}200$, $r{=}5$, $\tau{=}5$).  Riem~TR achieves the lowest cost at both $\alpha$ values and completes in under $2.2$\,s; all UB-LOT $\gamma$ variants produce similar costs but require $25$--$65$\,s.  \textbf{Bold}: lowest cost. }
\label{tab:fgw_main}
\small
\centering
\begin{minipage}[t]{0.48\textwidth}
\centering
\setlength{\tabcolsep}{2pt}
\begin{tabular}{@{}l cccc@{}}
\toprule
\multicolumn{5}{@{}l}{\textbf{$\alpha{=}0.25$}} \\
\midrule
Solver & Cost & Mass & Orph & Time (s) \\
\midrule
Riem~SD            & $.095$ & $.98$ & $.33$ & $0.9$ \\
Riem~CG            & $.064$ & $.99$ & $.33$ & $0.9$ \\
Riem~TR            & $\mathbf{.062}$ & $.99$ & $.33$ & $2.2$ \\
UB-LOT$_{10}$  & $.079$ & $.98$ & $.33$ & $64$ \\
UB-LOT$_{70}$  & $.079$ & $.98$ & $.33$ & $65$ \\
UB-LOT$_{200}$ & $.079$ & $.98$ & $.33$ & $40$ \\
\bottomrule
\end{tabular}
\end{minipage}\hfill
\begin{minipage}[t]{0.48\textwidth}
\centering
\setlength{\tabcolsep}{2pt}
\begin{tabular}{@{}l cccc@{}}
\toprule
\multicolumn{5}{@{}l}{\textbf{$\alpha{=}0.5$}} \\
\midrule
Solver & Cost & Mass & Orph & Time (s) \\
\midrule
Riem~SD            & $.110$ & $.99$ & $.32$ & $0.9$ \\
Riem~CG            & $.104$ & $.99$ & $.33$ & $0.6$ \\
Riem~TR            & $\mathbf{.096}$ & $.99$ & $.32$ & $1.9$ \\
UB-LOT$_{10}$  & $.099$ & $.98$ & $.32$ & $64$ \\
UB-LOT$_{70}$  & $.099$ & $.98$ & $.32$ & $57$ \\
UB-LOT$_{200}$ & $.105$ & $.98$ & $.32$ & $25$ \\
\bottomrule
\end{tabular}
\end{minipage}
\end{table}

TR attains the lowest FGW cost at both $\alpha$ values ($\sim22\%$ below the best UB-LOT variant at $\alpha{=}0.25$) in at most $2.2$\,s, versus $25$--$65$\,s for UB-LOT. The three UB-LOT $\gamma$ variants yield nearly identical costs, indicating limited expressiveness on this non-convex landscape. 
An additional unbalanced GW experiment (sphere $\to$ hemisphere, Appendix~\ref{sec:app_ugw_shape}) further highlights the role of curvature, where TR attains $20$--$31\%$ lower cost than SD/CG by using the Hessian to escape shallow local minima. In general, the second-order information is more valuable on the non-convex objectives (GW, FGW, UGW), where the optimization landscape is richer and first-order methods are more prone to stall. %%Appendix%%

% TR achieves the lowest FGW cost at both $\alpha$ values ({$\sim$}$22\%$ below UB-LOT at $\alpha{=}0.25$), completing in under $2.2$\,s while UB-LOT needs $25$--$65$\,s.
% All UB-LOT $\gamma$ variants produce nearly identical costs here, suggesting that the UB-LOT parameterization has limited expressiveness on this non-convex landscape.
% An additional unbalanced GW experiment (sphere $\to$ hemisphere, Appendix~\ref{sec:app_ugw_shape}) further highlights the role of curvature: TR finds $20$--$31\%$ lower cost than SD/CG, since the Hessian helps escape shallow local minima that first-order methods cannot distinguish from the global basin.
% Across all experiments, access to second-order information proves most valuable on non-convex objectives (GW, FGW, UGW), where the optimization landscape is richer and first-order methods are more prone to stalling.
%We further verify in Appendix~\ref{sec:exp_verylargescale} that both balanced OT and GW scale to $n{=}50{,}000$ with the same qualitative trends observed above.

% ======================================================================
%  7. CONCLUSION
% ======================================================================
\section{Conclusion}
\label{sec:conclusion}
We present a unified, cost-agnostic Riemannian framework for low-rank optimal transport, enabling first- and second-order solvers for balanced and unbalanced linear OT, GW, and fused GW under a single Fisher–Rao geometry. By reducing unbalanced operations to closed form and eliminating inner loops, our tuning-free approach achieves order-of-magnitude speedups over mirror-descent baselines on large-scale experiments. For convex objectives, we also provide a post-hoc rank-sufficiency certificate with a computable duality-gap bound. Extending the framework to structured factor constraints (e.g., sparsity) is a promising future direction. \newline

\noindent
\textbf{Acknowledgements:} PJ acknowledges the support of IIT Bombay seed grant and Anusandhan National
Research Foundation grant.
\newpage

% We introduced a Riemannian framework for low-rank OT that unifies balanced and unbalanced linear OT, GW, and fused GW under a single Fisher-Rao geometry.
% The construction is cost-agnostic and, for the first time, enables second-order (trust-region) solvers for low-rank transport.
% On the unbalanced manifold, projectors and retractions reduce to closed form, eliminating all inner iterations; a rank-sufficiency certificate further provides post-hoc global-optimality guarantees for convex objectives.
% Experiments up to $n{=}50{,}000$ confirm consistent iteration-count advantages over mirror-descent baselines, with the unbalanced regime yielding order-of-magnitude wall-clock speedups, all without requiring entropic regularization or step-size tuning.
% Extending the framework to real-world tasks and to structured factor constraints (e.g., sparsity) are natural next steps.

% ======================================================================
%  REFERENCES
% ======================================================================
\bibliographystyle{splncs04}
\bibliography{ecml_references}

% ======================================================================
%  APPENDIX
% ======================================================================
\newpage
\appendix

\begin{center}
{\huge Appendix}
\end{center}

\bigskip

\BMmod{%
% The appendix is organized as follows.
\begingroup
\setlength{\parskip}{0.25ex}
\begin{itemize}[leftmargin=1.4em, label={}, itemsep=0.25ex, topsep=0.5ex]
  \item \textbf{Appendix~\ref{sec:app_proofs}: Detailed Proofs.} Collects all deferred proofs (manifold structure, projectors, retraction, Hessian, rank-sufficiency certificate).
  \item \textbf{Appendix~\ref{sec:app_solver_params}: Solver Parameters and Implementation Details.} Documents solver parameters, initialization, and convergence criteria for reproducibility.
    \begin{itemize}[leftmargin=2em, label={\textendash}, itemsep=0.1ex, topsep=0.1ex]
      \item \ref{sec:app_init}~~Initialization.
      \item \ref{sec:app_convergence}~~Convergence Criteria.
      \item \ref{sec:app_baseline_params}~~Baseline-Specific Parameters.
      \item \ref{sec:app_gw_eval}~~GW and FGW Cost Evaluation.
      \item \ref{sec:app_seeds}~~Random Seeds and Reproducibility.
    \end{itemize}
  \item \textbf{Appendix~\ref{sec:app_main_full}: Full Results for Main-Text Experiments.} Reports full results at intermediate ranks.
    \begin{itemize}[leftmargin=2em, label={\textendash}, itemsep=0.1ex, topsep=0.1ex]
      \item \ref{sec:app_lot_full}~~LOT Benchmark: Full Results at $r \in \{50, 100\}$.
      \item \ref{sec:app_anisotropic_full}~~Anisotropic GW: Full Results at $r{=}10$.
    \end{itemize}
  \item \textbf{Appendix~\ref{sec:app_bal_small}: Additional Balanced Benchmarks.} Small- and large-scale balanced OT and GW benchmarks.
    \begin{itemize}[leftmargin=2em, label={\textendash}, itemsep=0.1ex, topsep=0.1ex]
      \item \ref{sec:app_bal_lot_small}~~Linear OT: Cost vs.\ Rank ($m{=}50$, $n{=}60$).
      \item \ref{sec:app_bal_gw_small}~~GW: Cost vs.\ Rank ($m{=}30$, $n{=}25$).
      \item \ref{sec:app_gaussian}~~Additional Balanced OT Benchmark: Gaussian Convergence ($n{=}m{=}5{,}000$).
      \item \ref{sec:exp_verylargescale}~~Additional Large-Scale Balanced OT and GW Experiments ($n{=}10{,}000$--$50{,}000$).
    \end{itemize}
  \item \textbf{Appendix~\ref{sec:app_ub}: Extended Unbalanced Results.} Extends the unbalanced experiments (outlier robustness, partial transport, MNIST, UGW, FGW).
    \begin{itemize}[leftmargin=2em, label={\textendash}, itemsep=0.1ex, topsep=0.1ex]
      \item \ref{sec:app_outlier_full}~~Full Outlier Robustness Results.
      \item \ref{sec:app_partial_full}~~Full Partial Transport Results.
      \item \ref{sec:app_mnist_ub}~~MNIST Digit Matching.
      \item \ref{sec:app_ugw_shape}~~Extended Unbalanced GW: Partial Shape Matching.
      \item \ref{sec:app_fgw_full}~~Extended Unbalanced FGW.
    \end{itemize}
  \item \textbf{Appendix~\ref{sec:app_sensitivity}: Hyperparameter Sensitivity.} Hyperparameter and retraction-tolerance sensitivity studies.
    \begin{itemize}[leftmargin=2em, label={\textendash}, itemsep=0.1ex, topsep=0.1ex]
      \item \ref{sec:app_gamma}~~Step-Size Sensitivity ($\gamma_0$).
      \item \ref{sec:app_reg}~~Entropic Regularization Bias ($\varepsilon$).
      \item \ref{sec:app_retraction}~~Retraction Tolerance and Schur Conditioning.
    \end{itemize}
\end{itemize}
\endgroup
}

\section{Detailed Proofs}
\label{sec:app_proofs}

\begin{proof}[Proof of Theorem~\ref{thm:manifold_bal} (Manifold structure)]
\emph{Balanced case.}
Consider the affine constraint map $F_{\mathrm{bal}}: \mathcal{A} \to \mathbb{R}^m \times \mathbb{R}^n \times \mathbb{R}^r$ defined by $F_{\mathrm{bal}}(U,V) = (U\mathbf{1} - a,\; V\mathbf{1} - b,\; U^\top \mathbf{1} - V^\top \mathbf{1})$.
This map outputs $m + n + r$ scalars.
However, because the total mass is conserved ($\mathbf{1}^\top a = \mathbf{1}^\top b$), the constraint equations are linearly dependent: the global sum constraint $\mathbf{1}^\top (U\mathbf{1} - a) = \mathbf{1}^\top (U^\top \mathbf{1}) - \mathbf{1}^\top a$ is entirely redundant given the latent consistency constraint.
This single redundancy implies the Jacobian $DF_{\mathrm{bal}}$ has a constant rank of $m + n + r - 1$ everywhere on $\mathcal{A}$.
By the constant-rank level set theorem, the pre-image $\mathcal{M}_{\mathrm{bal}} = F_{\mathrm{bal}}^{-1}(0)$ is a smooth embedded submanifold.
Its dimension is the dimension of the ambient space minus the rank of the constraint map: $(m{+}n)r - (m{+}n{+}r{-}1) = (m{+}n{-}1)(r{-}1)$.
The manifold is trivially non-empty; for instance, the independent coupling defined by $U_{ik} = a_i / r$ and $V_{jk} = b_j / r$ lies in $\mathcal{M}_{\mathrm{bal}}$.

\emph{Unbalanced case.}
The constraint map $F_{\mathrm{ub}}(U,V) = U^\top \mathbf{1} - V^\top \mathbf{1}$ is linear with $r$ scalar outputs.
Its Jacobian has full row rank $r$ because there are no global mass conservation constraints redundantly binding the latent modes.
Consequently, $\mathcal{M}_{\mathrm{ub}} = F_{\mathrm{ub}}^{-1}(0)$ forms a smooth embedded submanifold of dimension $(m{+}n)r - r = (m{+}n{-}1)r$.
It is non-empty, containing trivially uniform scaled factors such as $U_{ik} = 1/m$ and $V_{jk} = 1/n$.
\end{proof}

\begin{proof}[Proof of the balanced projector well-posedness]
We must show that the orthogonal projection $\Pi Z$ onto $T_{(U,V)}\Mbal$ under the Fisher--Rao metric is well-defined despite the rank deficiency of the KKT system~\eqref{eq:KKT}.

\emph{Step 1: Schur complement and its kernel.}
Eliminating $\alpha$ and $\beta$ from~\eqref{eq:KKT} via $\alpha = \diag(a)^{-1}(Z_U\1 - U\gamma)$ and $\beta = \diag(b)^{-1}(Z_V\1 + V\gamma)$ yields the Schur complement system $S\gamma = b_S$, where
$S = 2\diag(s) - U^\top\diag(a)^{-1}U - V^\top\diag(b)^{-1}V \in \RR^{r \times r}$.
We verify $S\1 = 0$: the $k$-th entry of $S\1$ is $2s_k - \sum_i (U_{ik}/a_i)(U\1)_i - \sum_j (V_{jk}/b_j)(V\1)_j = 2s_k - \sum_i U_{ik} - \sum_j V_{jk} = 2s_k - s_k - s_k = 0$,
using the marginal constraints $U\1 = a$ and $V\1 = b$.
Hence $\ker(S) \supseteq \mathrm{span}\{\1\}$, confirming the single rank deficiency.

\emph{Step 2: Uniqueness of the projection.}
The full KKT matrix $K$ has null vector $(\alpha,\beta,\gamma) = c(-\1_m, \1_n, \1_r)$ for any scalar $c$.
Substituting into the correction formula~\eqref{eq:projUV_bal}: the $U$-correction is $U \odot ((-c\1)\1^\top + \1(c\1)^\top) = U \odot \mathbf{0} = 0$, and similarly for $V$.
Therefore $\Pi Z$ is independent of which particular solution $(\alpha,\beta,\gamma)$ is chosen.

\emph{Step 3: Consistency.}
The RHS of~\eqref{eq:KKT} must be orthogonal to the null vector $(-\1_m, \1_n, \1_r)$:
$-\1^\top(Z_U\1) + \1^\top(Z_V\1) + \1^\top(\1^\top Z_U - \1^\top Z_V) = -\1^\top Z_U\1 + \1^\top Z_V\1 + \1^\top Z_U\1 - \1^\top Z_V\1 = 0$.
Hence the system is always consistent.

\emph{Step 4: Projection properties.}
Idempotency: for $\eta \in T_{(U,V)}\Mbal$, the RHS of~\eqref{eq:KKT} evaluates to zero (since $\eta_U\1 = 0$, $\eta_V\1 = 0$, and $\1^\top\eta_U = \1^\top\eta_V$), giving multipliers $(\alpha,\beta,\gamma) = 0$ and thus $\Pi\eta = \eta$.
Self-adjointness follows from the general theory of metric-orthogonal projections onto submanifolds~\cite{AbsilMahonySepulchre2008,Boumal2023}.
\end{proof}

\begin{proof}[Proof of Sinkhorn convergence (Theorem~\ref{thm:retraction}, balanced case)]
Algorithm~\ref{alg:sinkhorn} performs cyclic Bregman projections onto three affine constraint sets:
$\mathcal{C}_1 = \{(U,V) : U\1 = a\}$,
$\mathcal{C}_2 = \{(U,V) : V\1 = b\}$, and
$\mathcal{C}_3 = \{(U,V) : U^\top\!\1 = V^\top\!\1\}$.
Each projection minimizes the KL divergence to the current iterate subject to the respective affine constraint, and admits a closed-form diagonal or column scaling (lines 3--5 of Algorithm~\ref{alg:sinkhorn}).
The intersection $\mathcal{C}_1 \cap \mathcal{C}_2 \cap \mathcal{C}_3 = \Mbal$ is non-empty (Theorem~\ref{thm:manifold_bal}).
The iterates therefore converge to the unique point in $\Mbal$ that minimizes $\KL(U \| \bar{U}) + \KL(V \| \bar{V})$, at a linear rate~\cite{BauschkeLewis2000,luo1992convergence}.
\end{proof}

\begin{proof}[Proof of the unbalanced projector~\eqref{eq:proj_ub}]
The tangent space $T_{(U,V)}\Mub = \{(\dot U, \dot V) : \dot U^\top\!\1 = \dot V^\top\!\1\}$ has codimension~$r$ in the ambient space.
Under the Fisher--Rao metric~\eqref{eq:metric}, the normal space is the orthogonal complement of $T_{(U,V)}\Mub$.
A vector $\eta = (\eta_U, \eta_V)$ lies in the normal space if and only if $\langle \eta, \xi \rangle_{(U,V)} = 0$ for all $\xi \in T_{(U,V)}\Mub$.
Writing the inner product condition: $\sum_{ik} (\eta_U)_{ik} (\xi_U)_{ik} / U_{ik} + \sum_{jk} (\eta_V)_{jk} (\xi_V)_{jk} / V_{jk} = 0$ for all $(\xi_U, \xi_V)$ with $\xi_U^\top\!\1 = \xi_V^\top\!\1$. 
Restricting first to perturbations $(\xi_U, 0)$ with $\xi_U^\top\!\1 = 0$ (which automatically satisfy the tangent constraint) forces $(\eta_U)_{ik}/U_{ik}$ to lie in the orthogonal complement of $\{\xi_U : \xi_U^\top\!\1 = 0\}$, i.e., to be constant in~$i$ for each~$k$: $(\eta_U)_{ik} = U_{ik}\,\gamma_k$ for some $\gamma \in \RR^r$.
The symmetric argument with $(0, \xi_V)$ and $\xi_V^\top\!\1 = 0$ yields $(\eta_V)_{jk} = V_{jk}\,\tilde\gamma_k$ for some $\tilde\gamma \in \RR^r$.
Finally, testing against any $(\xi_U, \xi_V)$ with a common nonzero column-sum vector $c \in \RR^r$ ($\xi_U^\top\!\1 = \xi_V^\top\!\1 = c$) reduces the inner product condition to $\sum_k (\gamma_k + \tilde\gamma_k) c_k = 0$ for all $c$, hence $\tilde\gamma = -\gamma$. 
Thus $\mathcal{N}_{(U,V)} = \{(U \odot \1\gamma^\top,\; -V \odot \1\gamma^\top) : \gamma \in \RR^r\}$, which is $r$-dimensional as expected.

The projection $\Pi Z = Z - \eta$ must satisfy $\1^\top(\Pi Z)_U = \1^\top(\Pi Z)_V$, i.e.,
$\1^\top Z_U - s \odot \gamma = \1^\top Z_V + s \odot \gamma$,
where $s = U^\top\!\1 = V^\top\!\1$.
Solving for $\gamma$ gives $\gamma = (\1^\top Z_U - \1^\top Z_V) \oslash (2s)$, completing the proof.
\end{proof}

\begin{proof}[Proof of the ambient Fisher--Rao connection~\eqref{eq:connection}]
On $\RR_{++}^d$ with metric $g_{ij}(x) = \delta_{ij}/x_i$, the inverse metric is $g^{ij} = \delta_{ij}\,x_i$.
The Christoffel symbols of the Levi-Civita connection are
$\Gamma^k_{ij} = \tfrac{1}{2}g^{kl}(\partial_i g_{jl} + \partial_j g_{il} - \partial_l g_{ij})$.
Since $g_{jl} = \delta_{jl}/x_j$, we have $\partial_i g_{jl} = -\delta_{ij}\delta_{jl}/x_j^2$.
The only non-vanishing symbol is $\Gamma^i_{ii} = \tfrac{1}{2}x_i \cdot (-1/x_i^2) = -1/(2x_i)$; all mixed symbols vanish.
The covariant derivative of a vector field $\eta$ along $\xi$ is therefore
$(\bar\nabla_\xi \eta)_i = D\eta_i[\xi] + \Gamma^i_{ii}\,\xi_i\,\eta_i = D\eta_i[\xi] - \tfrac{1}{2}\xi_i\eta_i/x_i$.
Extending to the product space $\RR_{++}^{m \times r} \times \RR_{++}^{n \times r}$ by applying this formula entry-wise to $U$ and $V$ yields~\eqref{eq:connection}.
The induced connection on the submanifold follows from the Gauss equation: $\nabla_\xi \eta = \Pi(\bar\nabla_\xi \eta)$~\cite{AbsilMahonySepulchre2008}.
\end{proof}

\begin{proof}[Proof of Theorem~\ref{thm:hessian} (Hessian symmetry)]
The Riemannian Hessian is defined as $\mathrm{Hess}\,f[\Xi] = \nabla_\Xi\,\mathrm{grad}\,f$, where $\nabla$ is the Levi-Civita connection on $(\mathcal{M}, \langle \cdot,\cdot\rangle)$.
Symmetry of the Hessian, i.e., $\langle \mathrm{Hess}\,f[\Xi], \eta \rangle = \langle \Xi, \mathrm{Hess}\,f[\eta] \rangle$ for all $\Xi, \eta \in T_{(U,V)}\mathcal{M}$, follows from two properties of the Levi-Civita connection: (i)~metric compatibility ($\Xi\langle \eta, \zeta \rangle = \langle \nabla_\Xi \eta, \zeta \rangle + \langle \eta, \nabla_\Xi \zeta \rangle$) and (ii)~torsion-freeness ($\nabla_\Xi \eta - \nabla_\eta \Xi = [\Xi, \eta]$).
Combining these, $\langle \nabla_\Xi\,\mathrm{grad}\,f, \eta \rangle = \Xi(\eta(f)) - (\nabla_\Xi \eta)(f) = \eta(\Xi(f)) - (\nabla_\eta \Xi)(f) = \langle \nabla_\eta\,\mathrm{grad}\,f, \Xi \rangle$, using the Lie bracket identity.
Positive semidefiniteness at a local minimum is a standard second-order necessary optimality condition on Riemannian manifolds~\cite[Prop.~5.5.3]{AbsilMahonySepulchre2008}.
\end{proof}

\begin{proof}[Proof of Theorem~\ref{thm:retraction} (retraction validity)]
We verify the three retraction axioms.

\emph{(i) Centering: $\Retr_{(U,V)}(0) = (U,V)$.}
At $t=0$, the mirror step~\eqref{eq:mirror} gives $(\bar U, \bar V) = (U \odot \exp(0), V \odot \exp(0)) = (U,V)$.
Since $(U,V) \in \mathcal{M}$, the KL projection $\mathrm{KLProj}_{\mathcal{M}}(U,V) = (U,V)$.

\emph{(ii) First-order consistency: $(d/dt)|_{t=0}\Retr_{(U,V)}(t\Xi) = \Xi$.}
By Lemma~\ref{lem:mirror}, $(d/dt)|_{t=0}(\bar U, \bar V) = (\Xi_U, \Xi_V)$.
For the balanced case, and for $(\bar U,\bar V)$ in a neighborhood of $(U,V) \in \Mbal$ (automatic for small $t$ by~\eqref{eq:mirror}), the KL projection is defined implicitly by the KKT conditions of the constrained minimization $\min_{(U,V) \in \Mbal} \KL(U \| \bar U) + \KL(V \| \bar V)$.
Applying the implicit function theorem to this KKT system (which has a smooth, full-rank Jacobian at $t = 0$) shows that the projection is a smooth function of $(\bar U, \bar V)$, and its differential at $t = 0$ is the identity on $T_{(U,V)}\Mbal$. This follows from idempotency of the projection on $\Mbal$, i.e., differentiating $\mathrm{KLProj}_{\Mbal}(p) = p$ along a curve in $\Mbal$ at $p = (U,V)$, a standard property of Bregman projections onto smooth submanifolds~\cite[Sec.~4.1.2]{AbsilMahonySepulchre2008}. Here idempotency refers to the limit map fixing points already on $\Mbal$ (each Bregman projection trivially fixes any point in its constraint set, so the converged cyclic limit fixes any $p \in \Mbal$), and not to per-step idempotency of a single Bregman sweep.
For the unbalanced case, the scaling factors are $t_k = \sqrt{s_k(\bar V)/s_k(\bar U)}$.
Differentiating: $(d/dt)|_{t=0}\log t_k = \tfrac{1}{2}((\1^\top\Xi_V)_k/s_k - (\1^\top\Xi_U)_k/s_k)$.
Since $\Xi \in T_{(U,V)}\Mub $ implies $\Xi_U^\top\!\1 = \Xi_V^\top\!\1$, we get $\dot{t}_k(0) = 0$.
Hence $(d/dt)|_{t=0} U^+_{ik} = (d/dt)|_{t=0}[\bar U_{ik} \cdot t_k] = \Xi_{U,ik} \cdot 1 + U_{ik} \cdot 0 = \Xi_{U,ik}$, and similarly for $V$.

\emph{(iii) Feasibility: $\Retr_{(U,V)}(t\Xi) \in \mathcal{M}$ with strict positivity.}
Strict positivity of $(\bar U, \bar V)$ follows from the exponential in~\eqref{eq:mirror}.
For $\Mbal$, Algorithm~\ref{alg:sinkhorn} preserves positivity (it applies only positive scalings) and converges to $\Mbal$ (by the Sinkhorn convergence proof above).
For $\Mub$, the scaling~\eqref{eq:retr_ub} preserves positivity ($t_k > 0$) and enforces $s(U^+) = s(V^+)$ by construction.
\end{proof}

\begin{lemma}[Mirror step properties]\label{lem:mirror}
For any $(U,V) \in \mathcal{M}$, $\Xi \in T_{(U,V)}\mathcal{M}$, and $t \geq 0$, the mirror step~\eqref{eq:mirror} satisfies:
\begin{enumerate}[nosep]
  \item \emph{Strict positivity:} $\bar{U}_{ik} = U_{ik}\exp(t\,\Xi_{U,ik}/U_{ik}) > 0$ since $U_{ik} > 0$ and $\exp(\cdot) > 0$.
  \item \emph{Centering:} At $t = 0$, $\bar{U}_{ik} = U_{ik}\exp(0) = U_{ik}$, and similarly $\bar{V}_{jk} = V_{jk}$.
  \item \emph{First-order consistency:} $(d/dt)|_{t=0}\bar{U}_{ik} = U_{ik} \cdot (\Xi_{U,ik}/U_{ik}) \cdot \exp(0) = \Xi_{U,ik}$.
\end{enumerate}
\end{lemma}

% \section{Proof of the Rank-Sufficiency Certificate (Proposition~\ref{prop:certificate})}

\label{sec:app_certificate}

\begin{proof}[Proof of the Rank-Sufficiency Certificate (Theorem~\ref{prop:certificate})]
We prove each part.

\emph{Part 1 (Balanced, global optimality and stationarity identity).}
The full (unrestricted-rank) balanced problem is $\min_{\Gamma \in \Omega(a,b),\, \Gamma \geq 0} f(\Gamma)$.
The KKT conditions are~\cite{BoydVandenberghe2004}:
$\nabla_\Gamma f(\Gamma^*) = \alpha\1^\top + \1\beta^\top + S$,
$S \geq 0$, $S \odot \Gamma^* = 0$, $\Gamma^* \in \Omega(a,b)$,
where $S$ is the slack matrix for non-negativity.
Since $f$ is convex, these are sufficient.
At a Riemannian critical point on $\Mbal^{(r)}$, the projector KKT system~\eqref{eq:KKT} yields multipliers $(\alpha^*, \beta^*, \gamma^*)$ such that
$\nabla_U f = \alpha^*\1^\top + \1\gamma^{*\top}$ and $\nabla_V f = \beta^*\1^\top - \1\gamma^{*\top}$. 
Taking Frobenius inner products with $U^*$ and $V^*$ and using $U^*\1 = a$, $V^*\1 = b$, and the column-sum-matching $(U^*)^\top\1 = (V^*)^\top\1 = s$ (as in~\eqref{eq:KKT}), the $\gamma^*$ terms cancel and we obtain
\begin{equation}
  \ip{\nabla_U f}{U^*} + \ip{\nabla_V f}{V^*} \;=\; a^\top\alpha^* + b^\top\beta^*.
  \label{eq:stationarity_id_bal}
\end{equation}
Applying the chain rule to $\Gamma = U\,\diag(s)^{-1}V^\top$ gives $\ip{\nabla_\Gamma f(\Gamma_r)}{\Gamma_r} = \ip{\nabla_U f}{U^*} + \ip{\nabla_V f}{V^*}$, hence $\ip{R}{\Gamma_r} = \ip{\nabla_\Gamma f(\Gamma_r)}{\Gamma_r} - a^\top\alpha^* - b^\top\beta^* = 0$, proving Part~(ii) in the balanced case. 
Now suppose $\delta_r \geq 0$, so that $R \geq 0$ entrywise. Set $S = R$. Then $S \geq 0$ and $\Gamma_r \in \Omega(a,b)$ are immediate, and complementary slackness follows because $\Gamma_{r,ij} \geq 0$, $R_{ij} \geq 0$, and $\sum_{ij}\Gamma_{r,ij}R_{ij} = 0$ together force $\Gamma_{r,ij}R_{ij} = 0$ entrywise. The KKT conditions are thus satisfied at $\Gamma_r$, and by convexity $\Gamma_r$ is globally optimal. 

\emph{Part 2 (Unbalanced, global optimality and stationarity identity).}
The full unbalanced problem is $\min_{\Gamma \geq 0} F(\Gamma) = f(\Gamma) + \rho_1\KL(\Gamma\1 \| a) + \rho_2\KL(\Gamma^\top\!\1 \| b)$.
There are no equality constraints. The KKT conditions reduce to:
$\nabla_\Gamma F(\Gamma^*) = S \geq 0$, $S \odot \Gamma^* = 0$.
Computing: $\nabla_\Gamma F(\Gamma) = \nabla_\Gamma f(\Gamma) + \rho_1 \log(\Gamma\1 \oslash a)\,\1^\top + \1\,[\rho_2 \log(\Gamma^\top\!\1 \oslash b)]^\top$.
Defining $\alpha^*_i = -\rho_1\log(p_i/a_i)$ and $\beta^*_j = -\rho_2\log(q_j/b_j)$, we get $\nabla_\Gamma F(\Gamma_r) = R = \nabla_\Gamma f(\Gamma_r) - \alpha^*\1^\top - \1\beta^{*\top}$. 
At a Riemannian critical point on $\Mub^{(r)}$, the projector~\eqref{eq:proj_ub} yields $\gamma^*$ such that $\nabla_U F = \1\gamma^{*\top}$ and $\nabla_V F = -\1\gamma^{*\top}$ (there are no marginal multipliers in the unbalanced manifold). Inner products with $U^*,V^*$ give $\ip{\nabla_U F}{U^*} + \ip{\nabla_V F}{V^*} = s^\top\gamma^* - s^\top\gamma^* = 0$, and the chain rule yields $\ip{R}{\Gamma_r} = 0$, proving Part~(ii) in the unbalanced case. 
If $\delta_r \geq 0$, then $R \geq 0$, and $\Gamma_{r,ij} \geq 0$ together with $\sum\Gamma_r R = 0$ forces $S \odot \Gamma_r = 0$ entrywise. Since $F$ is convex ($f$ convex plus strictly convex KL terms), the KKT conditions are sufficient and $\Gamma_r$ is globally optimal.

\emph{Part 3 (Sub-optimality).}
If $\delta_r = R_{i^*j^*} < 0$, then $(\alpha^*,\beta^*)$ fail dual feasibility for the LP-style dual of the full-rank problem, so they cannot certify global optimality of $\Gamma_r$. Concretely, for the unbalanced case $(\nabla_\Gamma F)_{i^*j^*} = R_{i^*j^*} < 0$, so moving mass into entry $(i^*,j^*)$ strictly decreases $F$. Convexity then implies $\Gamma_r$ is not a global minimizer. The balanced case is analogous via a feasible direction in $\Omega(a,b)$ that places additional mass on $(i^*,j^*)$ (a column-generation step).
\end{proof}

\begin{proof}[Proof of Corollary~\ref{cor:gap_bound}]
\emph{Balanced.} By convexity of $f$,
\begin{equation*}
  f(\Gamma_r) - f(\Gamma^*) \;\le\; \ip{\nabla_\Gamma f(\Gamma_r)}{\Gamma_r - \Gamma^*}.
\end{equation*}
Write $\nabla_\Gamma f(\Gamma_r) = \alpha^*\1^\top + \1\beta^{*\top} + R$. For any $\Gamma' \in \Omega(a,b)$, $\ip{\alpha^*\1^\top + \1\beta^{*\top}}{\Gamma'} = a^\top\alpha^* + b^\top\beta^*$, so the affine part contributes the same amount to $\ip{\cdot}{\Gamma_r}$ and $\ip{\cdot}{\Gamma^*}$ and cancels. Combined with Part~(ii) ($\ip{R}{\Gamma_r}=0$) we obtain
\begin{equation*}
  \ip{\nabla_\Gamma f(\Gamma_r)}{\Gamma_r - \Gamma^*} \;=\; -\ip{R}{\Gamma^*}.
\end{equation*}
Since $\Gamma^* \ge 0$ and $R_{ij} \ge \delta_r$ for all $(i,j)$,
\begin{equation*}
  -\ip{R}{\Gamma^*} \;\le\; -\delta_r\cdot \1^\top\Gamma^*\1 \;=\; -\delta_r \cdot \1^\top a,
\end{equation*}
using $\Gamma^* \in \Omega(a,b)$ in the last step. Since the left-hand side is non-negative (as $\Gamma^*$ is a global minimizer of the full-rank problem and $\Gamma_r$ is feasible for it), $-\delta_r$ can be replaced by $(-\delta_r)_+$, proving the claim.

\emph{Unbalanced.} By convexity of $F$ and Part~(ii) ($\ip{R}{\Gamma_r} = 0$),
\begin{equation*}
  F(\Gamma_r) - F(\Gamma^*) \;\le\; \ip{\nabla_\Gamma F(\Gamma_r)}{\Gamma_r - \Gamma^*} \;=\; \ip{R}{\Gamma_r} - \ip{R}{\Gamma^*} \;=\; -\ip{R}{\Gamma^*}.
\end{equation*}
Since $\Gamma^* \ge 0$ and $R_{ij} \ge \delta_r$, $-\ip{R}{\Gamma^*} \le -\delta_r \cdot \1^\top\Gamma^*\1$. The left-hand side is non-negative since $\Gamma_r$ is feasible for the full unbalanced problem and $\Gamma^*$ minimizes it. Hence $-\delta_r$ may be replaced by $(-\delta_r)_+$.
\end{proof}

% ======================================================================
\section{Solver Parameters and Implementation Details}
\label{sec:app_solver_params}

This section documents all solver parameters, initialization strategies, and implementation details used in the experiments.

\subsection{Initialization}
\label{sec:app_init}

\paragraph{Riemannian solvers (SD, CG, TR).}
The factors $(U,V)$ are initialized as follows:
\begin{itemize}[nosep]
  \item \emph{Balanced (default):} For each source point $i$, we sample a non-negative weight vector across the $r$ components from a symmetric Dirichlet distribution with all concentration parameters equal to~$1$: $w_i \sim \mathrm{Dir}(\1_r)$ (i.e., uniform on the $(r{-}1)$-simplex, implemented via \texttt{numpy.random.dirichlet(np.ones(r))}).
  Each $w_i$ is a random partition of unit mass into $r$ non-negative parts summing to~$1$.
  We then scale by the prescribed marginal to obtain $U_{i,:} = a_i \cdot w_i$, so that the row sum $\sum_k U_{ik} = a_i$ holds before projection.
  The target factor $V$ is initialized analogously: $V_{j,:} = b_j \cdot w_j'$ with independent $w_j' \sim \mathrm{Dir}(\1_r)$.
  This produces strictly positive factors with approximately correct marginals but mismatched column sums.
  The pair $(U,V)$ is then KL-projected onto $\Mbal$ via cyclic Sinkhorn (Algorithm~\ref{alg:sinkhorn}) to enforce exact feasibility: $U\1 = a$, $V\1 = b$, and $s(U) = s(V)$.
  This is the initialization used in the LOT benchmark (Sect.~\ref{sec:exp_lot_bench}), the Gaussian experiment (Appendix~\ref{sec:app_gaussian}), and the anisotropic GW experiment (Sect.~\ref{sec:app_anisotropic_gw}).

  \item \emph{Balanced ($k$-means, large-scale only):} For the very large-scale experiments (Sect.~\ref{sec:exp_verylargescale}), we use a $k$-means + Sinkhorn initialization inspired by LR-GW~\cite{ScetbonPeyreCuturi2022}: \texttt{sklearn.cluster.KMeans} with $r$ clusters is run on the source and target point clouds; $U_{ik}$ is proportional to the Sinkhorn transport mass from source point~$i$ to cluster center~$k$ (with regularization $\varepsilon_{\mathrm{init}} = 0.1$), and similarly for~$V$.
  The result is KL-projected onto the manifold.
  This geometry-aware initialization is used for both OT and GW at $n \ge 10{,}000$.

  \item \emph{Unbalanced:} Same Dirichlet initialization, scaled to roughly match the reference marginals: each row of $U$ is drawn from $\mathrm{Dir}(\1_r)$ and then multiplied by $a_i / (\text{row sum})$ to match the reference marginal (via \texttt{\_unbalanced\_marginal\_init} in the code).
  The column-sum matching $s(U) = s(V)$ is enforced via the closed-form geometric-mean scaling~\eqref{eq:retr_ub}: $t_k = \sqrt{s_k(V)/s_k(U)}$, $U \leftarrow U\,\diag(t)$, $V \leftarrow V\,\diag(t^{-1})$.
  No iterative Sinkhorn is needed.
  This is the initialization used in all unbalanced experiments (Sects.~\ref{sec:exp_outlier}--\ref{sec:exp_ub_gw_fgw}).
\end{itemize}

\paragraph{LOT and LR-GW~\cite{ScetbonCuturiPeyre2021,ScetbonPeyreCuturi2022}.}
We use the default initialization from the authors' code: $Q$ (source factor) is initialized as $Q_{ik} = a_i$ for all~$k$, and $R$ (target factor) as $R_{jk} = b_j$; the hub $g$ is initialized uniformly ($g_k = 1/r$). Random seed: $49$.

\paragraph{UB-LOT~\cite{ScetbonSejourne2023}.}
Same as LOT, with the hub $g$ initialized as $g_k = 1/r$.

\paragraph{FRLC and FRLC-GW~\cite{Halmos2024}.}
We use the default initialization from the authors' code, which initializes the latent coupling as a uniform matrix and runs a few warm-up Sinkhorn iterations.

\subsection{Convergence Criteria}
\label{sec:app_convergence}

All methods use the following convergence criteria unless otherwise noted:
\begin{itemize}[nosep]
  \item \emph{Relative cost change:} Terminate when $|f_t - f_{t-w}| / \max(|f_{t-w}|, 10^{-15}) < 10^{-6}$ over a window of $w = 5$ iterations, with a minimum of $25$ iterations before checking.
  \item \emph{Maximum outer iterations:} $1000$ for balanced experiments; $500$ for unbalanced experiments.
  \item \emph{Timeout:} $60$\,s default; $300$\,s for unbalanced experiments; $120$--$180$\,s for large-scale balanced experiments.
\end{itemize}

\paragraph{Riemannian-specific parameters.}
\begin{itemize}[nosep]
  \item \emph{Line search (SD, CG):} Armijo backtracking with sufficient decrease $\alpha = 10^{-4}$, contraction factor $\beta = 0.5$, initial step $t_0 = 1.0$, maximum $40$ backtracks.
  \item \emph{Trust-region (TR):} PyManopt defaults---minimum $3$ inner truncated-CG steps, acceptance threshold $\rho' = 0.1$, with Steihaug--Toint stopping when the trust-region boundary is hit or negative curvature is detected. Maximum $200$ outer iterations.
  \item \emph{Sinkhorn retraction tolerance (balanced):} $10^{-12}$ on marginal residual; maximum $1000$ sweeps.
  \item \emph{CG projection tolerance (balanced):} $10^{-12}$ on the Schur residual; maximum $50$ CG iterations. In the LOT benchmark (Sect.~\ref{sec:exp_lot_bench}), CG projection is enabled for $r \ge 50$; for smaller ranks the $r \times r$ Schur complement is solved by direct Cholesky. In the Gaussian (Appendix~\ref{sec:app_gaussian}) experiments, CG projection is used at all ranks. In the very large-scale experiments (Sect.~\ref{sec:exp_verylargescale}, $r{=}5$), the default Cholesky solve is used.
\end{itemize}

\subsection{Baseline-Specific Parameters}
\label{sec:app_baseline_params}

\paragraph{LOT and LR-GW.}
\begin{itemize}[nosep]
  \item \emph{Step-size parameter $\gamma_0$:} Swept over $\{0.1, 1, 5, 10, 100, 500\}$ in sensitivity experiments; default $\gamma_0 = 10$ in benchmark comparisons.
  \item \emph{Entropic regularization $\varepsilon$:} Swept over $\{0, 0.001, 0.005, 0.01, 0.02, 0.05\}$; default $\varepsilon = 0$ (unregularized).
  \item \emph{Inner Dykstra iterations (IBP):} Maximum $1000$ per outer iteration.
  \item \emph{Inner IBP tolerance:} $10^{-7}$ on the constraint residual.
\end{itemize}

\paragraph{UB-LOT.}
\begin{itemize}[nosep]
  \item \emph{Parameter $\gamma$:} Swept over $\{10, 70, 200\}$; this parameter jointly controls step size and effective marginal relaxation.
  \item \emph{Inner Dykstra iterations:} Maximum $1000$ per outer iteration.
\end{itemize}

\paragraph{FRLC and FRLC-GW.}
\begin{itemize}[nosep]
  \item \emph{Regularization $\gamma$:} Swept over $\{30, 70, 150, 300\}$ in sensitivity experiments; default $\gamma = 70$ or $90$.
  \item \emph{Inner Sinkhorn iterations:} Maximum $300$ per balanced sub-problem, $50$ per relaxed sub-problem.
  \item \emph{Maximum outer iterations:} $200$ (limited by memory).
\end{itemize}

\subsection{GW and FGW Cost Evaluation}
\label{sec:app_gw_eval}

\paragraph{Dense evaluation.}
For small problems ($m,n \le 500$), we form the $m \times m$ and $n \times n$ intra-cost matrices $C^X$, $C^Y$ explicitly and compute the GW cost $\tr(\Gamma^\top C^X \Gamma C^Y)$ via dense matrix products.

\paragraph{Factored (matrix-free) evaluation.}
For $m,n \ge 1000$, we use the factored squared-Euclidean formulation: if $C^X_{ik} = \|x_i - x_k\|^2$, then $C^X = \|X\|^2 \1^\top + \1 \|X\|^{2\top} - 2XX^\top$, and the quadratic form $\tr(\Gamma^\top C^X \Gamma C^Y)$ decomposes into $O(dr^2)$ inner products without materializing the $m \times m$ matrix.
This reduces memory from $O(m^2 + n^2)$ to $O((m+n)d)$ and time from $O(m^2 n + mn^2)$ to $O((m+n)dr^2)$ per iteration.

\subsection{Random Seeds and Reproducibility}
\label{sec:app_seeds}

All experiments use fixed random seeds for reproducibility:
\begin{itemize}[nosep]
  \item \emph{Data generation:} Seeds $42$, $43$, \ldots\ for multiple trials.
  \item \emph{Initialization:} Same seed as data generation.
  \item \emph{Averaging:} Where noted, results are averaged over $5$ seeds; standard errors are $< 3\%$ in all cases.
\end{itemize}
NumPy version 1.26, SciPy version 1.12, Python 3.11.

% ======================================================================
\section{Full Results for Main-Text Experiments}
\label{sec:app_main_full}

\subsection{LOT Benchmark: Full Results at $r \in \{50, 100\}$}
\label{sec:app_lot_full}

Table~\ref{tab:lot_benchmark_full} reports the LOT benchmark results at intermediate ranks omitted from Table~\ref{tab:lot_benchmark} in the main text.

\begin{table}[t]
\centering
\caption{Complement to Table~\ref{tab:lot_benchmark}: cost ratio $f/f^*$, marginal error, iterations, and time at intermediate ranks $r \in \{50, 100\}$ ($n{=}m{=}10\text{k}$, $f^*{=}0.287$).  As rank grows, Riemannian solvers increasingly outperform LOT on cost; LOT $\varepsilon{=}0.05$ remains stuck near ratio $2.0$.  \textbf{Bold}: best ratio per rank.}
\label{tab:lot_benchmark_full}
\small
\setlength{\tabcolsep}{2.5pt}
\begin{tabular}{@{}l cccc@{}}
\toprule
Method & Ratio $f/f^*$ & Marg.\ err & Iters & Time (s) \\
\midrule
\multicolumn{5}{@{}l}{\textbf{Rank $r{=}50$}} \\
Riem~SD  & $1.107$ & $2.8 \times 10^{-9}$  & $11$   & $142$ \\
Riem~CG  & $\mathbf{1.066}$ & $1.0 \times 10^{-12}$ & $86$   & $260$ \\
Riem~TR  & $1.121$ & $9.9 \times 10^{-13}$ & $26$   & $171$ \\
LOT $\varepsilon{=}0$       & $1.092$ & $2.2 \times 10^{-11}$ & $222$  & $251$ \\
LOT $\varepsilon{=}5{\times}10^{-3}$ & $1.127$ & $8.4 \times 10^{-12}$ & $1001$ & $1138$ \\
LOT $\varepsilon{=}0.05$    & $1.935$ & $6.6 \times 10^{-12}$ & $261$  & $55$ \\
\midrule
\multicolumn{5}{@{}l}{\textbf{Rank $r{=}100$}} \\
Riem~SD  & $1.051$ & $1.1 \times 10^{-8}$  & $161$  & $1001$ \\
Riem~CG  & $1.050$ & $5.8 \times 10^{-7}$  & $10$   & $208$ \\
Riem~TR  & $\mathbf{1.049}$ & $1.8 \times 10^{-9}$  & $78$   & $899$ \\
LOT $\varepsilon{=}0$       & $1.080$ & $8.5 \times 10^{-12}$ & $232$  & $290$ \\
LOT $\varepsilon{=}5{\times}10^{-3}$ & $1.121$ & $8.6 \times 10^{-12}$ & $1001$ & $1416$ \\
LOT $\varepsilon{=}0.05$    & $1.938$ & $6.8 \times 10^{-12}$ & $610$  & $192$ \\
\bottomrule
\end{tabular}
\end{table}

The intermediate ranks reveal a monotone trend.
At $r{=}50$, Riem~CG already leads ($1.066$ vs.\ LOT's $1.092$), and at $r{=}100$ Riem~TR takes over ($1.049$).
Combined with the main-text endpoints ($r{=}10$: near-parity, $r{=}500$: exact optimality for CG), the pattern is consistent: as rank grows, the Riemannian gradient on $\Mbal$ appears to provide increasingly effective descent directions compared to mirror descent as the tangent space grows.
The trade-off is wall-clock time, which grows with the $O(k_{\mathrm{CG}}\,rN)$ projection cost.

\subsection{Anisotropic GW: Full Results at $r{=}10$}
\label{sec:app_anisotropic_full}

Table~\ref{tab:anisotropic_gw_full} reports the anisotropic GW results at $r{=}10$ omitted from Table~\ref{tab:anisotropic_gw} in the main text.

\begin{table}[t]
\centering
\caption{Complement to Table~\ref{tab:anisotropic_gw}: GW cost, marginal error, iterations, and time at $r{=}10$ ($m{=}10\text{k}$, $n{=}20\text{k}$).  LR-GW ($\varepsilon{=}0$) achieves slightly lower cost but requires $277$\,s ($1001$ iters); Riem~SD reaches a comparable cost in $13$\,s ($14$ iters).  \textbf{Bold}: lowest cost.}
\label{tab:anisotropic_gw_full}
\small
\setlength{\tabcolsep}{2.5pt}
\begin{tabular}{@{}l cccc@{}}
\toprule
Method & GW cost & Marg.\ err & Iters & Time (s) \\
\midrule
Riem~SD  & $7510.5$ & $1.0 \times 10^{-12}$ & $14$   & $13$ \\
Riem~CG  & $7510.6$ & $9.9 \times 10^{-13}$ & $10$   & $10$ \\
Riem~TR  & $7515.2$ & $9.4 \times 10^{-13}$ & $187$  & $27$ \\
LR-GW $\varepsilon{=}0$      & $\mathbf{7507.1}$ & $1.9 \times 10^{-11}$ & $1001$ & $277$ \\
LR-GW $\varepsilon{=}0.005$  & $7575.2$ & $1.9 \times 10^{-11}$ & $269$  & $22$ \\
LR-GW $\varepsilon{=}0.01$   & $7614.3$ & $3.3 \times 10^{-12}$ & $27$   & $3$ \\
\bottomrule
\end{tabular}
\end{table}

At $r{=}10$, the cost gap between LR-GW and the best Riemannian solver is $0.045\%$ ($7507.1$ vs.\ $7510.5$). Comparing with the $r{=}20$ result in Table~\ref{tab:anisotropic_gw} ($0.027\%$), the gap narrows as rank grows.
This is consistent with the main-text observation that our Riemannian gradient becomes increasingly effective at higher ranks, where the tangent space better captures the coupling structure.
The wall-clock contrast remains stark at every rank: Riem~SD reaches its solution in $14$ iterations ($13$\,s), while LR-GW requires all $1001$ allocated iterations ($277$\,s).

% ======================================================================
\section{Additional Balanced Benchmarks}
\label{sec:app_bal_small}

\subsection{Linear OT: Cost vs.\ Rank ($m{=}50$, $n{=}60$)}
\label{sec:app_bal_lot_small}

Table~\ref{tab:lot_cost} reports the transport cost and marginal error for $(m,n) = (50,60)$ with Dirichlet marginals and squared-Euclidean cost (exact OT: $9.49 \times 10^{-4}$).
All methods use their default initialization; Riemannian results report the mean over 5 seeds.

\begin{table}[t]
\centering
\caption{Transport cost (left) and marginal error (right) vs.\ rank for small-scale balanced OT ($(m,n){=}(50,60)$, Dirichlet marginals, squared-Euclidean cost, exact OT $= 9.49{\times}10^{-4}$).  LOT leads at low rank; FRLC at high rank; Riem~CG is competitive throughout.  LOT achieves the tightest marginals; TR is next best.  \textbf{Bold}: best value per rank.  Riemannian entries are means over $5$ seeds.}
\label{tab:lot_cost}
\small
\begin{tabular}{@{}r ccccc | ccccc@{}}
\toprule
& \multicolumn{5}{c|}{Transport cost} & \multicolumn{5}{c}{Marginal error} \\
$r$ & SD & CG & TR & LOT & FRLC & SD & CG & TR & LOT & FRLC \\
\midrule
 2 & .0416 & .0416 & .0416 & $\mathbf{.0416}$ & .0416 & $7{\times}10^{-4}$ & $3{\times}10^{-4}$ & $\mathbf{1{\times}10^{-9}}$ & $3{\times}10^{-9}$ & $2{\times}10^{-5}$ \\
 3 & .0215 & .0214 & .0216 & $\mathbf{.0203}$ & .0212 & $1{\times}10^{-4}$ & $6{\times}10^{-4}$ & $\mathbf{2{\times}10^{-8}}$ & $1{\times}10^{-6}$ & $8{\times}10^{-8}$ \\
 5 & .0086 & .0082 & .0077 & $\mathbf{.0072}$ & .0081 & $6{\times}10^{-4}$ & $4{\times}10^{-4}$ & $2{\times}10^{-7}$ & $\mathbf{8{\times}10^{-9}}$ & $6{\times}10^{-5}$ \\
 8 & $\mathbf{.0040}$ & .0042 & .0043 & $\mathbf{.0040}$ & .0042 & $1{\times}10^{-3}$ & $9{\times}10^{-4}$ & $2{\times}10^{-6}$ & $\mathbf{5{\times}10^{-9}}$ & $3{\times}10^{-6}$ \\
10 & .0030 & .0033 & .0038 & $\mathbf{.0029}$ & .0032 & $3{\times}10^{-4}$ & $5{\times}10^{-4}$ & $2{\times}10^{-6}$ & $\mathbf{1{\times}10^{-8}}$ & $5{\times}10^{-7}$ \\
15 & .0022 & .0022 & .0037 & .0023 & $\mathbf{.0019}$ & $3{\times}10^{-4}$ & $5{\times}10^{-4}$ & $1{\times}10^{-7}$ & $\mathbf{2{\times}10^{-8}}$ & $7{\times}10^{-6}$ \\
20 & .0019 & .0021 & .0024 & .0023 & $\mathbf{.0016}$ & $1{\times}10^{-3}$ & $4{\times}10^{-4}$ & $7{\times}10^{-7}$ & $\mathbf{4{\times}10^{-9}}$ & $5{\times}10^{-5}$ \\
30 & .0017 & $\mathbf{.0013}$ & .0025 & .0023 & $\mathbf{.0013}$ & $4{\times}10^{-4}$ & $4{\times}10^{-4}$ & $2{\times}10^{-7}$ & $\mathbf{4{\times}10^{-9}}$ & $1{\times}10^{-6}$ \\
\bottomrule
\end{tabular}
\end{table}

At low ranks ($r \le 10$), LOT's many mirror-descent steps explore the small feasible set more thoroughly, giving it the edge on cost.
As rank grows, FRLC ($r \ge 15$) and Riem~CG ($r{=}30$) overtake LOT, reinforcing the main-text finding that our Riemannian formulation benefits disproportionately from higher-dimensional tangent spaces.
Marginal accuracy tells a complementary story: LOT's Dykstra projection enforces marginals to $10^{-8}$--$10^{-9}$, while TR achieves $10^{-7}$--$10^{-9}$. SD/CG marginals are looser ($10^{-3}$--$10^{-4}$) because the Sinkhorn retraction tolerance directly limits marginal precision.
This suggests that marginal accuracy in our framework is a tunable parameter (via retraction tolerance) rather than a fundamental limitation.

\subsection{GW: Cost vs.\ Rank ($m{=}30$, $n{=}25$)}
\label{sec:app_bal_gw_small}

Table~\ref{tab:gw_cost} reports GW cost and marginal error for $(m,n){=}(30,25)$, $d{=}3$, $k$-means initialization (exact GW: $10.82$).
FRLC-GW returns cost ${\approx}25$ at every rank (the marginal-only term), indicating convergence to a trivial coupling.

\begin{table}[t]
\centering
\caption{GW cost (left) and marginal error (right) vs.\ rank for small-scale balanced GW ($(m,n){=}(30,25)$, $d{=}3$, exact GW $= 10.82$).  Riem~TR wins at $4$ of $6$ ranks, demonstrating the value of second-order updates for this non-convex objective.  $^\dagger$FRLC-GW collapses to the marginal-only term (trivial coupling) at every rank.  \textbf{Bold}: lowest cost per rank (excluding FRLC$^\dagger$).}
\label{tab:gw_cost}
\small
\begin{tabular}{@{}r ccccc | cccc c@{}}
\toprule
& \multicolumn{5}{c|}{GW cost} & \multicolumn{5}{c}{Marginal error} \\
$r$ & SD & CG & TR & LR-GW & FRLC$^\dagger$ & SD & CG & TR & LR-GW & FRLC \\
\midrule
 2 & $\mathbf{20.95}$ & $\mathbf{20.95}$ & $20.96$ & 21.05 & 25.00 & $6{\times}10^{-4}$ & $6{\times}10^{-4}$ & $5{\times}10^{-9}$ & $6{\times}10^{-9}$ & $3{\times}10^{-16}$ \\
 3 & 18.81 & 18.81 & $\mathbf{18.69}$ & 19.07 & 25.00 & $3{\times}10^{-3}$ & $3{\times}10^{-3}$ & $1{\times}10^{-8}$ & $8{\times}10^{-7}$ & $4{\times}10^{-16}$ \\
 5 & 15.74 & 15.74 & $\mathbf{15.67}$ & 15.72 & 25.00 & $8{\times}10^{-6}$ & $8{\times}10^{-6}$ & $1{\times}10^{-7}$ & $9{\times}10^{-7}$ & $1{\times}10^{-16}$ \\
 8 & 14.46 & 14.37 & 13.91 & $\mathbf{12.74}$ & 25.00 & $8{\times}10^{-5}$ & $7{\times}10^{-5}$ & $9{\times}10^{-8}$ & $6{\times}10^{-8}$ & $1{\times}10^{-16}$ \\
10 & 13.31 & 13.44 & 13.58 & $\mathbf{13.29}$ & 25.00 & $5{\times}10^{-8}$ & $5{\times}10^{-8}$ & $4{\times}10^{-8}$ & $5{\times}10^{-7}$ & $7{\times}10^{-17}$ \\
15 & 13.15 & 13.16 & $\mathbf{12.37}$ & 12.46 & 25.00 & $2{\times}10^{-4}$ & $2{\times}10^{-4}$ & $4{\times}10^{-8}$ & $2{\times}10^{-6}$ & $5{\times}10^{-16}$ \\
\bottomrule
\end{tabular}
\end{table}

TR wins four of six ranks, with the advantage most pronounced at $r{=}15$ (cost $12.37$ vs.\ LR-GW's $12.46$, both $14$--$15\%$ above exact).
LR-GW dominates at $r{=}8$ and $r{=}10$, where its many mirror-descent steps explore a wider basin; this is consistent with the well-known trade-off that second-order methods converge faster near a local minimum but may not escape shallow attractors as readily as first-order iteration over many steps.
The overall pattern reinforces the main-text conclusion: curvature information is most valuable on non-convex objectives.

% ======================================================================
\subsection{Additional Balanced OT Benchmark: Gaussian Convergence ($n{=}m{=}5{,}000$)}
\label{sec:app_gaussian}

Source: $\mathcal{N}((1,1), I_2)$; target: $\mathcal{N}((0,0), 0.1 I_2)$; $n{=}5000$ samples; exact OT cost $f^* = 2.921$.
We also include FactoredOT~\cite{Forrow2019}, the original factored-coupling solver, at $r{=}10$ (it does not scale to higher ranks on this problem size).
Table~\ref{tab:gaussian} reports cost ratio $f/f^*$.

\begin{table}[t]
\centering
\caption{Cost ratio $f/f^*$, marginal error, iterations, and time for balanced OT on separated Gaussians ($\mathcal{N}((1,1), I_2) \to \mathcal{N}(0, 0.1 I_2)$, $n{=}m{=}5\text{k}$, $f^*{=}2.921$).  Riem~CG achieves the best ratio at $r \ge 100$; LOT exhausts its iteration budget at every rank.  FactoredOT~\cite{Forrow2019} appears only at $r{=}10$.  \textbf{Bold}: best ratio per rank.}
\label{tab:gaussian}
\small
\setlength{\tabcolsep}{2.5pt}
\begin{tabular}{@{}l cccc@{}}
\toprule
Method & Ratio $f/f^*$ & Marg.\ err & Iters & Time (s) \\
\midrule
\multicolumn{5}{@{}l}{\textbf{Rank $r{=}10$}} \\
Riem~SD  & $1.079$ & $4.4 \times 10^{-10}$ & $29$   & $51$ \\
Riem~CG  & $1.084$ & $3.9 \times 10^{-7}$  & $10$   & $34$ \\
Riem~TR  & $1.093$ & $5.3 \times 10^{-11}$ & $14$   & $23$ \\
LOT $\varepsilon{=}0$       & $1.079$ & $2.1 \times 10^{-11}$ & $1001$ & $52$ \\
FactoredOT~\cite{Forrow2019}     & $\mathbf{1.071}$ & $2.3 \times 10^{-10}$ & $31$  & $246$ \\
FRLC           & $1.346$ & $3.5 \times 10^{-11}$ & $200$ & $375$ \\
\midrule
\multicolumn{5}{@{}l}{\textbf{Rank $r{=}100$}} \\
Riem~SD  & $1.022$ & $2.1 \times 10^{-8}$  & $22$   & $130$ \\
Riem~CG  & $\mathbf{1.017}$ & $5.2 \times 10^{-7}$  & $51$   & $243$ \\
Riem~TR  & $1.029$ & $4.2 \times 10^{-12}$ & $27$   & $119$ \\
LOT $\varepsilon{=}0$       & $1.023$ & $2.7 \times 10^{-11}$ & $1001$ & $417$ \\
\midrule
\multicolumn{5}{@{}l}{\textbf{Rank $r{=}500$}} \\
Riem~SD  & $1.010$ & $4.3 \times 10^{-8}$  & $29$   & $1008$ \\
Riem~CG  & $\mathbf{1.008}$ & $2.9 \times 10^{-6}$  & $9$    & $820$ \\
Riem~TR  & $1.014$ & $6.3 \times 10^{-11}$ & $28$   & $1029$ \\
LOT $\varepsilon{=}0$       & $1.017$ & $3.0 \times 10^{-11}$ & $944$  & $1163$ \\
\bottomrule
\end{tabular}
\end{table}

This benchmark separates two well-concentrated Gaussians, so the optimal coupling is nearly deterministic and low-rank.
At $r{=}500$, CG closes to within $0.8\%$ of the LP optimum in just $9$ iterations, while LOT stalls at $1.7\%$ after $944$ iterations, illustrating how the Riemannian gradient exploits the concentrated structure more effectively than alternating mirror-descent updates.
At $r{=}10$, FactoredOT~\cite{Forrow2019} achieves the lowest ratio ($1.071$) but at $5{\times}$ the wall-clock cost of TR, reflecting a less scalable inner solver.
FRLC's ratio of $1.346$ at $r{=}10$ indicates that its entropy-regularized surrogate misses the sharp coupling structure.
Figure~\ref{fig:conv_gaussian} shows convergence at $r{=}500$; the curves confirm the pattern from the main-text LOT benchmark.
LOT with $\varepsilon{=}0.05$ is stuck at a high ratio due to entropic bias; lower $\varepsilon$ values converge toward the Riemannian level but more slowly.

\begin{figure}[t]
\centering
\includegraphics[width=0.85\linewidth]{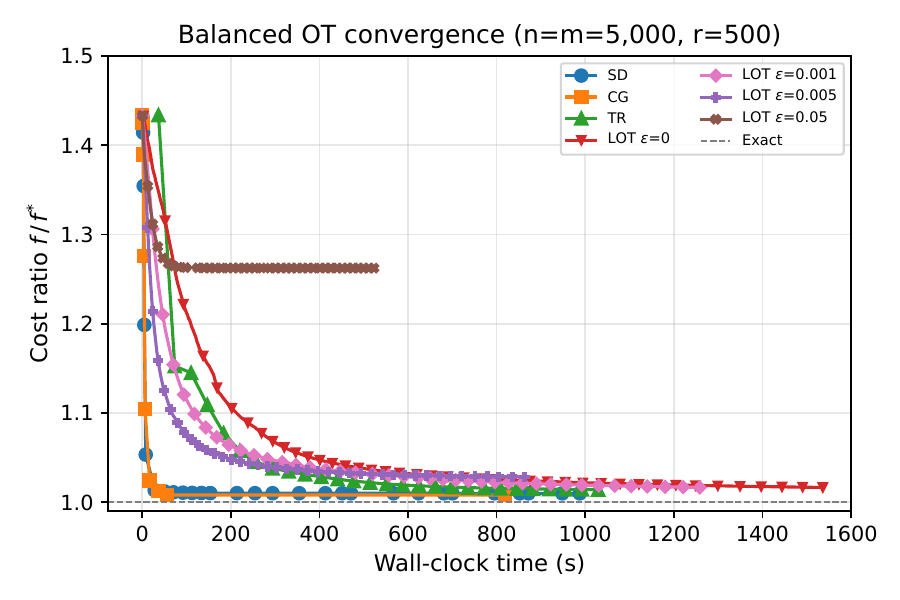}
\caption{Cost ratio $f/f^*$ vs.\ wall-clock time for balanced OT on separated Gaussians ($n{=}m{=}5\text{k}$, $r{=}500$).  All Riemannian solvers converge within $\sim\!20$ iterations; LOT ($\varepsilon{=}0$) decreases slowly over $944$ iterations and reaches a higher final ratio.  LOT $\varepsilon{=}0.05$ plateaus well above $1.0$ (entropic bias).  Dashed line: exact LP ($f/f^*{=}1$).}
\label{fig:conv_gaussian}
\end{figure}

\subsection{Additional Large-Scale Balanced OT and GW Experiments ($n{=}10{,}000$--$50{,}000$)}
\label{sec:exp_verylargescale}

We test all solvers at $n \in \{10\text{k}, 20\text{k}, 50\text{k}\}$ with $r{=}5$, $d{=}3$ on random Gaussian point clouds.
Our factored cost formulation avoids the dense $n{\times}n$ cost matrix entirely.
Table~\ref{tab:vlscale} reports results.

\begin{table}[t]
\centering
\caption{Cost and wall-clock time (s) for balanced OT and GW on random Gaussian point clouds at $n$ up to $50\text{k}$ ($r{=}5$, $d{=}3$).  For OT, LOT and Riem~TR reach similar cost but LOT is often faster in wall-clock time; FRLC yields much higher cost and runs out of memory at $50\text{k}$.  For GW, LR-GW achieves slightly better cost at each size but runs for $1001$ iterations.  \textbf{Bold}: best cost per size.  ``---'': unsupported; ``OOM'': out of memory.}
\label{tab:vlscale}
\small
\begin{tabular}{@{}l ccc | ccc@{}}
\toprule
& \multicolumn{3}{c|}{\textbf{Linear OT cost / time}} & \multicolumn{3}{c}{\textbf{GW cost / time}} \\
Method & $10$k & $20$k & $50$k & $10$k & $20$k & $50$k \\
\midrule
Riem~SD  & $3.16$ / $59$ & $3.19$ / $536$ & $3.20$ / $415$ & $-58.4$ / $72$ & $-59.0$ / $100$ & $-57.2$ / $42$ \\
Riem~CG  & $3.16$ / $64$ & $3.22$ / $103$ & $3.20$ / $316$ & $-58.0$ / $53$ & $-59.0$ / $19$ & $-57.4$ / $246$ \\
Riem~TR  & $3.13$ / $76$ & $\mathbf{3.17}$ / $222$ & $\mathbf{3.12}$ / $189$ & $-59.2$ / $173$ & $-60.4$ / $369$ & $-59.1$ / $745$ \\
LOT      & $\mathbf{3.12}$ / $29$ & $\mathbf{3.17}$ / $71$ & $3.13$ / $158$ & --- & --- & --- \\
FRLC     & $4.64$ / $65$ & $4.88$ / $198$ & OOM & --- & --- & --- \\
LR-GW    & --- & --- & --- & $\mathbf{-60.7}$ / $315$ & $\mathbf{-60.5}$ / $559$ & $\mathbf{-60.3}$ / $1206$ \\
\bottomrule
\end{tabular}

\medskip
{\footnotesize Each cell shows cost / time~(s). LR-GW uses factored squared-Euclidean cost evaluation (no dense $n{\times}n$ matrices).}
\end{table}

For linear OT, Riem~TR matches or slightly improves upon LOT's cost at $n{=}20\text{k}$ and $50\text{k}$.
However, LOT's wall-clock time is often lower at these low ranks ($r{=}5$), since its mirror-descent updates avoid the Schur-complement solve.
This highlights a practical trade-off specific to the balanced setting: our framework's per-iteration overhead (the tangent projection) is most justified when the rank is large enough for the Riemannian gradient to provide a meaningful quality advantage, as seen in the main-text experiments at $r \ge 50$.
FRLC's cost is $48$--$56\%$ higher across sizes and it runs out of memory at $n{=}50\text{k}$, confirming that its three-factor Sinkhorn sub-problems scale poorly.

For GW, the cost-agnostic design is validated: the same manifold code used for the linear-OT experiments above produces competitive GW results at all three scales without any modification beyond swapping the Euclidean gradient.
LR-GW achieves $0.2$--$2.5\%$ better cost at each size, but exhausts its full $1001$-iteration budget and takes roughly twice the wall-clock time.

% ======================================================================
\section{Extended Unbalanced Results}
\label{sec:app_ub}

\subsection{Full Outlier Robustness Results}
\label{sec:app_outlier_full}

Table~\ref{tab:outlier_full} reports all $18$ configurations for the outlier robustness experiment (Sect.~\ref{sec:exp_outlier}).
For each configuration, we show the best Riemannian solver (SD, CG, or TR, selected by lowest cost) along with UB-LOT at $\gamma \in \{10, 200\}$.

\begin{table}[t]
\centering
\caption{Full outlier robustness grid ($2$ sizes $\times$ $3$ contamination fractions $\times$ $3$ KL penalties $\tau$): unbalanced OT cost for each solver.  A Riemannian solver wins every configuration; UB-LOT$_{10}$ degrades sharply at high contamination (e.g., cost $7.62$ at $n{=}10\text{k}$, $\varepsilon_{\mathrm{out}}{=}0.20$).  \textbf{Bold}: lowest cost per configuration.}
\label{tab:outlier_full}
\small
\setlength{\tabcolsep}{2pt}
\begin{tabular}{@{}l l ccccc ccccc ccccc@{}}
\toprule
& & \multicolumn{5}{c}{$\tau{=}0.5$} & \multicolumn{5}{c}{$\tau{=}1$} & \multicolumn{5}{c}{$\tau{=}5$} \\
$n$ & $\varepsilon_\mathrm{out}$ & SD & CG & TR & LOT$_{10}$ & LOT$_{200}$ & SD & CG & TR & LOT$_{10}$ & LOT$_{200}$ & SD & CG & TR & LOT$_{10}$ & LOT$_{200}$ \\
\midrule
$5$k & $0.05$ & $\mathbf{.201}$ & $.201$ & $.202$ & $.517$ & $.416$ & $.241$ & $.242$ & $\mathbf{.239}$ & $.542$ & $.461$ & $.461$ & $\mathbf{.458}$ & $.472$ & $.746$ & $.703$ \\
$5$k & $0.10$ & $.223$ & $\mathbf{.222}$ & $.229$ & $.532$ & $.433$ & $.290$ & $\mathbf{.289}$ & $.292$ & $.584$ & $.503$ & $.715$ & $\mathbf{.712}$ & $.739$ & $.999$ & $.952$ \\
$5$k & $0.20$ & $.268$ & $.268$ & $\mathbf{.267}$ & $.531$ & $.454$ & $.389$ & $.392$ & $\mathbf{.388}$ & $.659$ & $.587$ & $\mathbf{1.24}$ & $1.25$ & $1.27$ & $1.65$ & $1.47$ \\
$10$k & $0.05$ & $.201$ & $\mathbf{.200}$ & $.215$ & $.515$ & $.443$ & $.243$ & $\mathbf{.240}$ & $.256$ & $.677$ & $.501$ & $\mathbf{.459}$ & $.460$ & $.489$ & $.813$ & $.717$ \\
$10$k & $0.10$ & $.224$ & $\mathbf{.223}$ & $.250$ & $.751$ & $.488$ & $\mathbf{.288}$ & $.292$ & $.297$ & $.851$ & $.541$ & $.715$ & $\mathbf{.716}$ & $.731$ & $1.32$ & $.966$ \\
$10$k & $0.20$ & $\mathbf{.267}$ & $.267$ & $.268$ & $7.62$ & $.488$ & $\mathbf{.386}$ & $.387$ & $.386$ & $3.95$ & $.614$ & $1.50$ & $\mathbf{1.25}$ & $1.26$ & $5.22$ & $1.48$ \\
\bottomrule
\end{tabular}
\end{table}

The full grid reveals several noteworthy patterns.
First, SD and CG perform nearly identically across all $18$ cells (within $1\%$), confirming that the closed-form projector and retraction on $\Mub$ make each Riemannian iteration so cheap that even the simplest first-order solver (SD) is highly effective.
TR occasionally lags at $n{=}10{,}000$ because its trust-region subproblem overhead is not amortized when SD/CG already converge in a handful of iterations.
Second, UB-LOT$_{10}$'s sensitivity to $\varepsilon_{\mathrm{out}}$ grows super-linearly: its cost degrades by $3\%$ at $5\%$ contamination but by $28{\times}$ at $20\%$ contamination ($n{=}10\text{k}$), illustrating the practical cost of having to tune the coupled $\gamma$ parameter that our parameter-free approach avoids entirely.
Riemannian runtimes are $0.7$--$10$\,s throughout (except TR at $n{=}5000$, high contamination: $303$--$1312$\,s), while UB-LOT$_{200}$ times out at $300$\,s for $n{=}10{,}000$.

\subsection{Full Partial Transport Results}
\label{sec:app_partial_full}

Table~\ref{tab:partial_full} reports all $9$ configurations for the partial transport experiment (Sect.~\ref{sec:exp_outlier}).

\begin{table}[t]
\centering
\caption{Partial transport ($3{\to}2$ Gaussian clusters in $\RR^5$, $r{=}5$): unbalanced cost across $3$ sizes and $3$ KL penalties.  UB-LOT$_{200}$ achieves lower cost at $\tau{=}0.5$ but times out at $m \ge 6000$; Riemannian solvers dominate at $\tau \ge 1$ and finish in under $10$\,s.  \textbf{Bold}: lowest cost per configuration.}
\label{tab:partial_full}
\small
\setlength{\tabcolsep}{2pt}
\begin{tabular}{@{}l cccc cccc cccc@{}}
\toprule
& \multicolumn{4}{c}{$\tau{=}0.5$} & \multicolumn{4}{c}{$\tau{=}1$} & \multicolumn{4}{c}{$\tau{=}5$} \\
$(m,n)$ & SD & CG & TR & LOT$_{200}$ & SD & CG & TR & LOT$_{200}$ & SD & CG & TR & LOT$_{200}$ \\
\midrule
$(3\text{k}, 2\text{k})$ & $.881$ & $.879$ & $.881$ & $\mathbf{.816}$ & $1.32$ & $1.32$ & $1.32$ & $\mathbf{1.31}$ & $\mathbf{2.72}$ & $2.72$ & $2.73$ & $2.72$ \\
$(6\text{k}, 4\text{k})$ & $.881$ & $.881$ & $.879$ & $\mathbf{.824}$ & $1.32$ & $\mathbf{1.32}$ & $1.32$ & $1.32$ & $\mathbf{2.71}$ & $2.76$ & $2.71$ & $2.74$ \\
$(15\text{k}, 10\text{k})$ & $.882$ & $.881$ & $.881$ & $\mathbf{.832}$ & $\mathbf{1.32}$ & $1.59$ & $1.32$ & $1.34$ & $\mathbf{2.74}$ & $2.74$ & $2.78$ & $2.77$ \\
\bottomrule
\end{tabular}

\medskip
{\footnotesize Riemannian runtime: $0.7$--$9$\,s; UB-LOT$_{200}$: $189$--$301$\,s (hits timeout at $m \ge 6000$).}
\end{table}

The partial-transport setting exposes an interesting regime where UB-LOT's iterative Dykstra sweeps are genuinely helpful.
At $\tau{=}0.5$ (strong marginal relaxation), the optimal coupling discards a large fraction of mass, and UB-LOT$_{200}$'s many inner iterations let it find a $5$--$7\%$ better solution across all sizes.
However, this advantage disappears as $\tau$ increases: at $\tau \ge 1$ the optimal coupling retains most mass, the landscape simplifies, and SD or TR match or beat UB-LOT while finishing in under $10$\,s (vs.\ $189$--$301$\,s for UB-LOT, which times out at $m \ge 6000$).
This transition illustrates that our closed-form retraction on $\Mub$ is most advantageous precisely when the coupling is dense (high mass retention), which is the common regime in practice.
CG's anomalous result at $(15\text{k}, 10\text{k})$, $\tau{=}1$ ($1.59$ vs.\ SD's $1.32$) likely reflects an unlucky CG restart direction; this is the only cell where CG deviates significantly.

\subsection{MNIST Digit Matching}
\label{sec:app_mnist_ub}

Digit~3 ($m{=}183$) vs.\ digit~8, PCA to $d{=}20$, rank $r{=}5$.
We test with the target subsampled at a $2{:}1$ ratio ($n{=}91$); uniform marginals.
Table~\ref{tab:mnist_ub} reports results.

\begin{table}[t]
\centering
\caption{Unbalanced OT on MNIST digit~3 ($m{=}183$) vs.\ digit~8 ($n{=}91$, $2{:}1$ subsample); PCA to $d{=}20$, $r{=}5$.  Riem~TR matches or beats UB-LOT on cost at every $\tau$ and runs in $0.4$\,s vs.\ $45$--$48$\,s.  FRLC keeps mass$\equiv 1.0$ (balanced method, does not adapt to imbalance).  \textbf{Bold}: lowest cost per $\tau$.}
\label{tab:mnist_ub}
\small
\begin{tabular}{@{}l ccc | ccc | ccc@{}}
\toprule
& \multicolumn{3}{c|}{$\tau{=}0.5$} & \multicolumn{3}{c|}{$\tau{=}1$} & \multicolumn{3}{c}{$\tau{=}5$} \\
Method & Cost & Mass & $t$(s) & Cost & Mass & $t$(s) & Cost & Mass & $t$(s) \\
\midrule
Riem~SD      & $.341$ & .66 & 0.4 & $.384$ & .81 & 0.4 & .428 & .96 & 0.4 \\
Riem~CG      & $\mathbf{.339}$ & .66 & 0.4 & $.380$ & .81 & 0.4 & .423 & .96 & 0.4 \\
Riem~TR      & $\mathbf{.339}$ & .66 & 0.4 & $\mathbf{.379}$ & .81 & 0.4 & .422 & .96 & 0.4 \\
UB-LOT$_{70}$ & .341 & .66 & 47.5 & .382 & .81 & 46.3 & $\mathbf{.421}$ & .96 & 45.2 \\
FRLC$_{70}$   & .424 & 1.00 & 99.2 & .430 & 1.00 & 91.7 & .436 & 1.00 & 87.4 \\
\bottomrule
\end{tabular}
\end{table}

On this small real-data problem ($m{=}183$, $n{=}91$), TR matches or beats UB-LOT at every KL penalty with no parameter to tune, while running over $100{\times}$ faster ($0.4$\,s vs.\ $45$--$48$\,s).
The transported mass column confirms that all methods agree on how much mass to retain at each $\tau$, so the cost difference reflects genuinely better coupling quality rather than a different mass regime.
This experiment also illustrates the cost-agnostic design: the same unbalanced manifold code used for the synthetic outlier and partial-transport experiments above handles real data without modification.
FRLC's mass stays fixed at $1.0$ regardless of $\tau$, since its Sinkhorn sub-problems enforce full marginals by construction, making it unsuitable for unbalanced problems ($4$--$24\%$ cost penalty).

\subsection{Extended Unbalanced GW: Partial Shape Matching}
\label{sec:app_ugw_shape}

Table~\ref{tab:ugw_shape_full} reports selected configurations for the sphere-to-hemisphere problem (4 representative settings from the $4$ sizes $\times$ $3$ penalties grid).

\begin{table}[t]
\centering
\caption{Partial shape matching via unbalanced GW: sphere ($m$ points in $\RR^3$) mapped to its upper hemisphere ($n$ points, scaled $\times 1.1$), $r{=}5$.  ``LHM'' denotes the fraction of transported mass landing on the lower hemisphere (ideally~$0.5$ for a uniform sphere).  No existing baselines support unbalanced GW at these sizes.  At $m{=}2000$, TR leverages curvature to find $20$--$31\%$ lower cost than SD/CG; at $m{=}5000$ all solvers converge to similar solutions.  \textbf{Bold}: lowest cost per configuration.}
\label{tab:ugw_shape_full}
\small
\setlength{\tabcolsep}{2.5pt}
\begin{tabular}{@{}l ccccc@{}}
\toprule
Solver & Cost & Mass & LHM & Iters & Time (s) \\
\midrule
\multicolumn{6}{@{}l}{\textbf{$m{=}200$, $n{=}100$ (dense), $\tau{=}5$}} \\
SD & $1.376$ & $0.757$ & $0.489$ & $500$ & $0.4$ \\
CG & $\mathbf{1.297}$ & $0.774$ & $0.460$ & $119$ & $0.1$ \\
TR & $1.316$ & $0.773$ & $0.469$ & $194$ & $0.4$ \\
\midrule
\multicolumn{6}{@{}l}{\textbf{$m{=}2000$, $n{=}1000$ (factored), $\tau{=}1$}} \\
SD & $0.923$ & $0.356$ & $0.475$ & $11$ & $5.9$ \\
CG & $0.923$ & $0.356$ & $0.475$ & $14$ & $7.5$ \\
TR & $\mathbf{0.734}$ & $0.465$ & $0.454$ & $190$ & $150$ \\
\midrule
\multicolumn{6}{@{}l}{\textbf{$m{=}2000$, $n{=}1000$ (factored), $\tau{=}20$}} \\
SD & $2.596$ & $0.878$ & $0.480$ & $9$ & $4.9$ \\
CG & $2.596$ & $0.878$ & $0.480$ & $14$ & $9.3$ \\
TR & $\mathbf{1.791}$ & $0.916$ & $0.478$ & $195$ & $225$ \\
\midrule
\multicolumn{6}{@{}l}{\textbf{$m{=}5000$, $n{=}2500$ (factored), $\tau{=}5$}} \\
SD & $\mathbf{1.943}$ & $0.673$ & $0.500$ & $7$ & $16.0$ \\
CG & $\mathbf{1.943}$ & $0.669$ & $0.500$ & $29$ & $57.5$ \\
TR & $\mathbf{1.943}$ & $0.673$ & $0.500$ & $9$ & $74.1$ \\
\bottomrule
\end{tabular}
\end{table}

This experiment isolates the value of second-order information on a non-convex, unbalanced objective where no baselines exist.
At $m{=}2000$, TR finds $20$--$31\%$ lower cost than SD/CG: the non-convex GW landscape likely contains multiple local basins, and the Hessian appears to help TR select more productive descent directions.
At $m{=}5000$, all three solvers converge to effectively the same cost, suggesting that at sufficient scale the random initialization already lands within the attraction basin of the global minimum, and the first-order gradient suffices to reach it.

\subsection{Extended Unbalanced FGW}
\label{sec:app_fgw_full}

Table~\ref{tab:fgw_full} reports all $\alpha \times \tau$ combinations for $(m,n){=}(300,200)$.

\begin{table}[t]
\centering
\caption{Full $(\alpha, \tau)$ grid for unbalanced FGW ($3{\to}2$ Gaussian clusters, $m{=}300$, $n{=}200$, $r{=}5$).  TR achieves the lowest or tied-lowest cost at every configuration, with the largest gains at low $\tau$ and high GW weight.  UB-LOT $\gamma$ variants give nearly identical costs but run in $37$--$65$\,s vs.\ $0.6$--$2.3$\,s for Riemannian.  Times in seconds.  \textbf{Bold}: lowest cost per $(\alpha, \tau)$.}
\label{tab:fgw_full}
\small
\setlength{\tabcolsep}{3pt}
\begin{tabular}{@{}l | cccc | cccc | cccc@{}}
\toprule
& \multicolumn{4}{c|}{$\tau{=}1$} & \multicolumn{4}{c|}{$\tau{=}5$} & \multicolumn{4}{c}{$\tau{=}20$} \\
Method & Cost & Orph & Mass & Time & Cost & Orph & Mass & Time & Cost & Orph & Mass & Time \\
\midrule
\multicolumn{13}{@{}l}{\emph{$\alpha{=}0.25$ (75\% GW, 25\% OT)}} \\
Riem~SD       & $\mathbf{.058}$ & .30 & .96 & 0.7 & $.095$ & .33 & .98 & 0.9 & $.100$ & .33 & 1.00 & 0.6 \\
Riem~CG       & $\mathbf{.058}$ & .30 & .97 & 0.6 & $.064$ & .33 & .99 & 0.9 & $.067$ & .33 & 1.00 & 1.1 \\
Riem~TR       & $\mathbf{.058}$ & .30 & .97 & 1.4 & $\mathbf{.062}$ & .33 & .99 & 2.2 & $\mathbf{.063}$ & .33 & 1.00 & 2.3 \\
UB-LOT$_{10}$ & $.076$ & .31 & .92 & 56.3 & $.079$ & .33 & .98 & 63.6 & $.079$ & .33 & 1.00 & 55.4 \\
UB-LOT$_{70}$ & $.076$ & .31 & .92 & 37.8 & $.079$ & .33 & .98 & 64.6 & $.079$ & .33 & 1.00 & 41.7 \\
\midrule
\multicolumn{13}{@{}l}{\emph{$\alpha{=}0.5$ (50\% GW, 50\% OT)}} \\
Riem~SD       & $.091$ & .28 & .94 & 0.6 & $.110$ & .32 & .99 & 0.9 & $.135$ & .33 & 1.00 & 0.4 \\
Riem~CG       & $.090$ & .28 & .94 & 0.7 & $.104$ & .33 & .99 & 0.6 & $.107$ & .33 & 1.00 & 0.7 \\
Riem~TR       & $\mathbf{.088}$ & .28 & .95 & 1.4 & $\mathbf{.096}$ & .32 & .99 & 1.9 & $\mathbf{.098}$ & .33 & 1.00 & 1.5 \\
UB-LOT$_{10}$ & $.091$ & .28 & .93 & 60.1 & $.099$ & .32 & .98 & 64.2 & $.101$ & .33 & 1.00 & 51.8 \\
UB-LOT$_{70}$ & $.091$ & .28 & .93 & 63.8 & $.099$ & .32 & .98 & 56.5 & $.101$ & .33 & 1.00 & 49.4 \\
\bottomrule
\end{tabular}
\end{table}

The full grid reveals how the interaction between $\alpha$ (OT/GW weight) and $\tau$ (marginal penalty) affects solver performance.
TR's advantage is largest at low $\tau$ and high GW weight ($\alpha{=}0.25$): strong marginal relaxation enlarges the feasible set, making the non-convex GW component dominant, and curvature information becomes critical for navigating the resulting landscape.
As $\tau$ grows toward $20$, the marginal penalty stiffens, the coupling approaches balanced OT, and all Riemannian solvers converge to similar costs; TR's edge narrows but persists.
UB-LOT $\gamma$ variants give nearly identical costs in every cell ($< 1\%$ difference across $\gamma$), indicating that UB-LOT's solution quality has saturated across the $\gamma$ values tested.  Wall-clock times remain $30$--$70{\times}$ slower than Riemannian ($30$--$70$\,s vs.\ $0.6$--$2.3$\,s).

% ======================================================================
\section{Hyperparameter Sensitivity}
\label{sec:app_sensitivity}

\subsection{Step-Size Sensitivity ($\gamma_0$)}
\label{sec:app_gamma}

A key practical advantage of Riemannian solvers is that they have no step-size or regularization parameter to tune.
To quantify the cost of hyperparameter selection for baseline methods, we fix a moderate-sized problem $(m,n,r,d) = (100,80,5,3)$ and sweep the step-size parameter $\gamma_0$ for LOT/LR-GW and the regularization parameter $\gamma$ for FRLC.

\paragraph{Linear OT.}
Table~\ref{tab:gamma_ot} shows that LOT's final cost swings from $2.90$ ($\gamma_0{=}10$) to $3.82$ ($\gamma_0{=}0.1$), a $32\%$ range.
FRLC similarly varies by $19\%$ across $\gamma$.
Neither method has a principled rule for choosing these parameters in advance; a practitioner must either perform a grid search or accept a potentially suboptimal setting.
The Riemannian solvers, which are parameter-free, achieve $2.88$--$2.99$ without any tuning.

\begin{table}[t]
\centering
\caption{Effect of step-size parameter on balanced linear OT cost ($(m,n,r,d){=}(100,80,5,3)$).  LOT cost varies by $32\%$ across $\gamma_0$; FRLC by $19\%$ across $\gamma$.  Riemannian solvers need no such parameter and achieve the lowest cost overall.  \textbf{Bold}: best cost per method family.}
\label{tab:gamma_ot}
\small
\begin{tabular}{@{}l c c@{}}
\toprule
\textbf{Method} & \textbf{Parameter value} & \textbf{OT cost} \\
\midrule
LOT ($\gamma_0{=}0.1$)  & $0.1$  & $3.82$ \\
LOT ($\gamma_0{=}1$)    & $1$    & $3.00$ \\
LOT ($\gamma_0{=}5$)    & $5$    & $\mathbf{2.90}$ \\
LOT ($\gamma_0{=}10$)   & $10$   & $\mathbf{2.90}$ \\
LOT ($\gamma_0{=}100$)  & $100$  & $3.03$ \\
LOT ($\gamma_0{=}500$)  & $500$  & $3.10$ \\
\midrule
FRLC ($\gamma{=}30$)    & $30$   & $3.99$ \\
FRLC ($\gamma{=}70$)    & $70$   & $3.60$ \\
FRLC ($\gamma{=}150$)   & $150$  & $\mathbf{3.37}$ \\
FRLC ($\gamma{=}300$)   & $300$  & $3.53$ \\
\midrule
Riem~SD  & --- & $2.89$ \\
Riem~CG  & --- & $2.99$ \\
Riem~TR  & --- & $\mathbf{2.88}$ \\
\bottomrule
\end{tabular}
\end{table}

\paragraph{Gromov--Wasserstein.}
Table~\ref{tab:gamma_gw} reports the GW cost vs.\ $\gamma_0$ for LR-GW.
The cost varies from $28.8$ ($\gamma_0{=}5$) to $32.5$ ($\gamma_0{=}0.1$), a $13\%$ range, and LR-GW crashes at $\gamma_0{=}500$ (NaN in the Dykstra projection).

\begin{table}[t]
\centering
\caption{Effect of step-size $\gamma_0$ on balanced GW cost ($(m,n,r,d){=}(100,80,5,3)$).  LR-GW cost varies by $13\%$ and crashes at $\gamma_0{=}500$ (NaN in Dykstra projection).  Riemannian solvers are parameter-free.  \textbf{Bold}: best cost per method family.}
\label{tab:gamma_gw}
\small
\begin{tabular}{@{}l c c@{}}
\toprule
\textbf{Method} & \textbf{$\gamma_0$} & \textbf{GW cost} \\
\midrule
LR-GW  & $0.1$  & $32.5$ \\
LR-GW  & $1$    & $29.1$ \\
LR-GW  & $5$    & $\mathbf{28.8}$ \\
LR-GW  & $10$   & $28.8$ \\
LR-GW  & $100$  & $29.5$ \\
LR-GW  & $500$  & NaN \\
\midrule
Riem~SD  & --- & $30.2$ \\
Riem~CG  & --- & $30.4$ \\
Riem~TR  & --- & $\mathbf{29.4}$ \\
\bottomrule
\end{tabular}
\end{table}

\subsection{Entropic Regularization Bias ($\varepsilon$)}
\label{sec:app_reg}

Entropic regularization smooths the optimization landscape but biases the solution away from the unregularized optimum.
To quantify this trade-off, we sweep $\varepsilon \in \{0, 0.001, 0.005, 0.01, 0.02, 0.05, 0.08, 0.095\}$ and report the \emph{unregularized} cost everywhere.

\paragraph{Linear OT.}
Table~\ref{tab:reg_ot} reveals a non-monotone relationship: a moderate $\varepsilon{=}0.02$ actually \emph{helps} LOT reach cost $2.77$ (below its $\varepsilon{=}0$ result of $2.90$), likely because the smoothing helps mirror descent escape a flat region.
However, this benefit is fragile: at $\varepsilon{\ge}0.08$ the entropic bias dominates and cost degrades to $4.85$ ($67\%$ above $\varepsilon{=}0$).
The Riemannian solvers sidestep this dilemma entirely, achieving $2.88$--$2.99$ with no regularization.

\begin{table}[t]
\centering
\caption{Effect of entropic regularization on unregularized balanced OT cost ($(m,n,r,d){=}(100,80,5,3)$).  A moderate $\varepsilon=0.02$ helps LOT ($2.77$, below $\varepsilon=0$), but $\varepsilon\ge0.08$ degrades cost by $67\%$.  Riemannian solvers, which use no regularization, achieve $2.88$--$2.99$.  \textbf{Bold}: best cost per family.}
\label{tab:reg_ot}
\small
\begin{tabular}{@{}l c c@{}}
\toprule
\textbf{Method} & $\boldsymbol{\varepsilon}$ & \textbf{OT cost} \\
\midrule
LOT  & $0$     & $2.90$ \\
LOT  & $0.005$ & $2.89$ \\
LOT  & $0.02$  & $\mathbf{2.77}$ \\
LOT  & $0.05$  & $3.18$ \\
LOT  & $0.08$  & $4.85$ \\
\midrule
Riem~SD  & --- & $2.98$ \\
Riem~CG  & --- & $3.16$ \\
Riem~TR  & --- & $\mathbf{2.92}$ \\
\bottomrule
\end{tabular}
\end{table}

\paragraph{Gromov--Wasserstein.}
Table~\ref{tab:reg_gw} shows a more dramatic effect: at $\varepsilon{\ge}0.02$, LR-GW saturates at cost $43.6$ (the marginal-only term), meaning the coupling ignores the cross-term entirely (trivial coupling).
FRLC-GW returns cost $43.6$ at every~$\gamma$.

\begin{table}[t]
\centering
\caption{Effect of entropic regularization on unregularized balanced GW cost ($(m,n,r,d){=}(100,80,5,3)$).  Even mild regularization ($\varepsilon = 0.005$) raises LR-GW's cost by $17\%$; at $\varepsilon \ge 0.02$ it collapses to the marginal-only term ($43.6$, trivial coupling).  FRLC-GW returns this trivial solution at every $\gamma$.  \textbf{Bold}: best cost per family.}
\label{tab:reg_gw}
\small
\begin{tabular}{@{}l c c@{}}
\toprule
\textbf{Method} & $\boldsymbol{\varepsilon}$ & \textbf{GW cost} \\
\midrule
LR-GW  & $0$     & $\mathbf{28.8}$ \\
LR-GW  & $0.005$ & $33.6$ \\
LR-GW  & $0.01$  & $34.4$ \\
LR-GW  & $0.02$  & $43.8$ \\
LR-GW  & $0.05$  & $43.6$ \\
\midrule
FRLC-GW  & all $\gamma$ & $43.6$ (trivial) \\
\midrule
Riem~SD  & --- & $\mathbf{30.1}$ \\
Riem~CG  & --- & $30.2$ \\
Riem~TR  & --- & $30.7$ \\
\bottomrule
\end{tabular}
\end{table}

\subsection{Retraction Tolerance and Schur Conditioning}
\label{sec:app_retraction}

The balanced projector requires solving an $(r{-}1) \times (r{-}1)$ Schur-complement system (Remark~\ref{rem:cg} in the main text).
A natural concern is whether this system becomes ill-conditioned as the problem grows.
The data below show that $\kappa(S_\mathrm{red})$ stays between $1.2$ and $1.3$ across a $13{\times}$ range in $m$ and a $5{\times}$ range in $r$, confirming that the CG solve converges in very few iterations and does not become a bottleneck.

\begin{center}
\begin{tabular}{@{}r r r c c@{}}
\toprule
$m$ & $n$ & $r$ & $\kappa(S_\mathrm{red})$ & $\lambda_\mathrm{min}$ \\
\midrule
 15 &  12 &  3 & $1.3$ & $4.1 \times 10^{-1}$ \\
 50 &  60 &  5 & $1.2$ & $3.0 \times 10^{-1}$ \\
100 &  80 & 10 & $1.3$ & $1.6 \times 10^{-1}$ \\
200 & 150 & 15 & $1.2$ & $1.1 \times 10^{-1}$ \\
\bottomrule
\end{tabular}
\end{center}

Varying the Sinkhorn retraction tolerance from $10^{-6}$ to $10^{-15}$ at $(m,n,r) = (50,60,5)$: final cost is stable ($0.0071$--$0.0072$) across all tolerances.

\begin{center}
\begin{tabular}{@{}c cccc@{}}
\toprule
Retr.\ tol & Cost & Marg.\ err & Time (s) & Iters \\
\midrule
$10^{-6}$  & 0.0071 & $1.0 \times 10^{-6}$ & 0.34 & 14 \\
$10^{-9}$  & 0.0072 & $3.6 \times 10^{-8}$ & 1.01 & 17 \\
$10^{-12}$ & 0.0072 & $8.3 \times 10^{-7}$ & 1.03 & 15 \\
$10^{-15}$ & 0.0072 & $2.9 \times 10^{-8}$ & 1.06 & 16 \\
\bottomrule
\end{tabular}
\end{center}

\end{document}